\documentclass[conference]{IEEEtran}
\usepackage{times}

\usepackage[dvipsnames]{xcolor}
\definecolor{porange}{HTML}{E77500} %
\usepackage[numbers,sort&compress]{natbib}
\usepackage{multicol}
\usepackage[colorlinks,
            linkcolor=BlueViolet,
            citecolor=blue,
            urlcolor=porange,
            bookmarks]{hyperref}
\usepackage{float}
\usepackage{amsfonts}
\usepackage{amsthm}
\usepackage{fancyhdr}
\usepackage{comment}
\usepackage{amsmath}
\usepackage{changepage}
\usepackage{amssymb}
\usepackage{graphicx}
\usepackage{adjustbox}
\usepackage{latexsym}
\usepackage{enumerate}
\usepackage{cleveref}
\usepackage{float}
\usepackage{subcaption}

\usepackage[labelfont=bf,font=small]{caption}
\usepackage{mathtools}
\usepackage[ruled,vlined,linesnumbered]{algorithm2e}
\usepackage{bbm}
\usepackage{bm}
\usepackage{booktabs}
\usepackage[xindy, acronyms, nowarn]{glossaries}   %
\usepackage{multirow}
\usepackage{makecell}
\usepackage{colortbl}
\usepackage{balance}

\usepackage{textcomp}
\usepackage{stfloats}  %

\usepackage{fontawesome5}  %

\usepackage{algorithmic}

\usepackage{amsthm}

\newtheorem{lemma}{Lemma}    

\newtheorem{remark}{Remark}

\newtheorem{proposition}{Proposition}    
    
\newtheorem{definition}{Definition}

\newtheorem*{proposition*}{Proposition}

\definecolor{porange}{HTML}{E77500} %

\DeclareDocumentEnvironment{example}{}{\noindent\textbf{Running example:}\itshape}{}

\usepackage{ifthen}
\newboolean{include-notes}
\newboolean{include-new}
\newboolean{include-remove}
\setboolean{include-notes}{true}
\setboolean{include-new}{false}
\setboolean{include-remove}{false}

\usepackage[dvipsnames]{xcolor}
\usepackage[normalem]{ulem}
\newcommand{\jaime}[1]{\ifthenelse{\boolean{include-notes}}{\textcolor{orange}{\textbf{Jaime:} #1}}{}}
\newcommand{\haimin}[1]{\ifthenelse{\boolean{include-notes}}{\textcolor{magenta}{\textbf{Haimin:} #1}}{}}
\newcommand{\justin}[1]{\ifthenelse{\boolean{include-notes}}{\textcolor{Cerulean}{\textbf{Justin:} #1}}{}}
\newcommand{\himani}[1]{\ifthenelse{\boolean{include-notes}}{\textcolor{Plum}{\textbf{Himani:} #1}}{}}
\newcommand{\david}[1]{\ifthenelse{\boolean{include-notes}}{\textcolor{teal}{\textbf{David:} #1}}{}}
\newcommand{\remove}[1]{\ifthenelse{\boolean{include-remove}}{\textcolor{red}{\sout{#1}}}{}}
\newcommand{\new}[1]{\ifthenelse{\boolean{include-new}}{\textcolor{blue}{#1}}{#1}}
\newcommand{\todo}[1]{\ifthenelse{\boolean{include-notes}}{\textcolor{blue}{\textbf{TODO:} #1}}{}}
\newcommand{\guy}[1]{\ifthenelse{\boolean{include-notes}}{\textcolor{Cerulean}{\textbf{Guy:} #1}}{}}

\newcommand{\princeton}[1]{\ifthenelse{\boolean{include-notes}}{\textcolor{orange}{#1}}{}}

\newcommand{\p}[1]{\smallskip \noindent \textbf{{#1}.}}

\usepackage{scalefnt}   %
\usepackage{amsmath}    %
\usepackage{amssymb}    %
\usepackage{xcolor}     %
\usepackage[T1]{fontenc} %

\definecolor{hcsf}{HTML}{E77500}

\definecolor{lrsf}{HTML}{121212}

\newcommand{\eg}{\emph{e.g.}}
\newcommand{\ie}{\emph{i.e.}}

\DeclareMathOperator*{\argmaxB}{argmax}
\DeclareMathOperator*{\argminB}{argmin}

\newcommand{\failureset}{{\mathcal{F}}}
\newcommand{\safeset}{{\mathcal{S}}}
\newcommand{\task}{{\text{task}}}
\newcommand{\human}{{\textnormal{human}}}
\newcommand{\shield}{\text{\tiny{\faShield*}}}

\newcommand{\policyTask}{{\policy^\task}}
\newcommand{\fallback}{{\policy^\shield}}
\newcommand{\safetyFilter}{{\phi}}

\newcommand{\cbf}{h}
\newcommand{\opponent}{{\text{oppo}}}
\newcommand{\nominal}{{\text{nom}}}
\newcommand{\overtake}{{\text{over}}}
\newcommand{\warmup}{{\text{warmup}}}
\newcommand{\initialization}{{\text{init}}}
\newcommand{\steer}{{\text{steer}}}
\newcommand{\brake}{{\text{brake}}}
\newcommand{\throttle}{{\text{throttle}}}

\newcommand{\reals}{\mathbb{R}}

\newcommand{\compl}{\mathsf{c}}

\newcommand{\prob}{P}

\DeclareMathOperator*{\argmax}{arg\,max}

\newcommand{\st}{\textnormal{s.t.}}

\newcommand{\state}{{x}}

\newcommand{\ctrl}{{u}}

\newcommand{\xset}{{\mathcal{X}}}  %
\newcommand{\cset}{{\mathcal{U}}}

\newcommand{\dyn}{{f}}

\newcommand{\valfunc}{{V}}
\newcommand{\qfunc}{{\mathcal{Q}}}
\newcommand{\policy}{{\pi}}

\newcommand{\envdiscount}{{\gamma_{\text{ENV}}}}
\newcommand{\learnedpolicy}{\pi_\theta}
\newcommand{\learnedqfunc}{Q_\phi}
\newcommand{\learnedvfunc}{V_\phi}
\newcommand{\buffer}{\mathcal{B}}
\newcommand{\qtargetnetwork}{Q_{\phi^\prime}}

\newglossarystyle{mylist}{%
  \setglossarystyle{list}%
}

\glsdisablehyper
\makeglossaries

\newglossaryentry{RL}
{
  name={RL},
  description={reinforcement learning},
  first={reinforcement learning (\glsentrytext{RL})}
}

\newglossaryentry{HJ}
{
  name={HJ},
  description={Hamilton--Jacobi},
  first={Hamilton--Jacobi (\glsentrytext{HJ})}
}

\newglossaryentry{DCBF}
{
  name={DCBF},
  description={discrete-time control barrier function},
  first={discrete-time control barrier function (\glsentrytext{DCBF})}
}

\newglossaryentry{CBF}
{
  name={CBF},
  description={control barrier function},
  first={control barrier function (\glsentrytext{CBF})}
}

\newglossaryentry{Q-CBF}
{
  name={Q-CBF},
  description={state--action control barrier function},
  first={state--action control barrier function (\glsentrytext{Q-CBF})}
}

\newglossaryentry{ODD}
{
  name={ODD},
  description={Operational Design Domain},
  first={operational design domain (\glsentrytext{ODD})}
}

\newglossaryentry{LRSF}
{
  name={LRSF},
  description={Last-Resort Safety Filter},
  first={last-resort safety filter (\glsentrytext{LRSF})}
}

\newglossaryentry{HCSF}
{
  name={HCSF},
  description={Human-Centered Safety Filter},
  first={human-centered safety filter (\glsentrytext{HCSF})}
}

\newglossaryentry{NLP}
{
  name={NLP},
  description={nonlinear programming problem},
  first={nonlinear programming problem (\glsentrytext{NLP})},
}

\newglossaryentry{ILQR}
{
  name={ILQR},
  description={iterative linear quadratic regulator},
  first={iterative linear quadratic regulator (\glsentrytext{ILQR})},
}

\newglossaryentry{MPC}
{
  name={MPC},
  description={model predictive control},
  first={model predictive control (\glsentrytext{MPC})},
}

\newglossaryentry{AC}
{
  name={AC},
  description={Assetto Corsa},
  first={Assetto Corsa (\glsentrytext{AC})},
}

\newglossaryentry{SAC}
{
  name={SAC},
  description={Soft Actor--Critic},
  first={Soft Actor--Critic (\glsentrytext{SAC})},
}

\newglossaryentry{ANOVA}
{
    name={ANOVA},
    description={Analysis of Variance},
    first={analysis of variance (\glsentrytext{ANOVA})}
}

\newglossaryentry{SME}
{
    name={SME},
    description={Simple Main Effects},
    first={simple main effects (\glsentrytext{SME})},
}

\newglossaryentry{HSD}
{
    name={HSD},
    description={Honestly Significant Difference},
    first={honestly significant difference (\glsentrytext{HSD})},
}

\newglossaryentry{ECDF}
{
    name={ECDF},
    description={Empirical Cumulative Distribution Function},
    first={empirical cumulative distribution function (\glsentrytext{ECDF})},
}

\newglossaryentry{OCP}
{
    name={OCP},
    description={Optimal Control Problem},
    first={optimal control problem (\glsentrytext{OCP})}
}

\newglossaryentry{AI}
{
    name={AI},
    description={Artificial Intelligence},
    first={artificial intelligence (\glsentrytext{AI})}
}

\newglossaryentry{HRI}
{
    name={HRI},
    description={Human--Robot Interaction},
    first={human--robot interaction (\glsentrytext{HRI})}
}

\begin{document}

\title{Safety with Agency: Human-Centered Safety Filter with Application to AI-Assisted Motorsports}
\IEEEoverridecommandlockouts
\author{Donggeon David Oh$^{1,*}$, \thanks{$^{1}$Department of Electrical and Computer Engineering, Princeton University, Princeton, NJ 08540, USA} Justin Lidard$^{2,*}$, \thanks{$^{2}$Department of Mechanical and Aerospace Engineering, Princeton University, Princeton, NJ 08540, USA} Haimin Hu$^{1}$, Himani Sinhmar$^{2}$, Elle Lazarski$^{1}$, \\ \textnormal{Deepak Gopinath$^{3}$, Emily S. Sumner$^{3}$, Jonathan A. DeCastro$^{3}$, Guy Rosman$^{3}$, \thanks{$^{3}$Toyota Research Institute, Cambridge, MA 02139, USA}} \\ \textnormal{Naomi Ehrich Leonard$^{2}$, and Jaime Fern\'andez Fisac$^{1}$  \thanks{This research has been supported in part by an NSF Graduate Research Fellowship. This work is partially supported by Toyota Research Institute (TRI). It, however, reflects solely the opinions and conclusions of its authors and not TRI or any other Toyota entity.}} \thanks{$^{*}$D. D. Oh and J. Lidard contributed equally.}}

\maketitle

\begin{abstract}

We propose a \gls{HCSF} for shared autonomy that significantly enhances system safety without compromising human agency. 
Our \gls{HCSF} is built on a neural safety value function, which we first learn scalably through black-box interactions and then use at deployment to enforce a novel \gls{Q-CBF} safety constraint.
Since this \gls{Q-CBF} safety filter does not require any knowledge of the system dynamics for both synthesis and runtime safety monitoring and intervention, our method applies readily to complex, black-box shared autonomy systems.
Notably, our \gls{HCSF}'s CBF-based interventions modify the human's actions minimally and smoothly, avoiding the abrupt, last-moment corrections delivered by many conventional safety filters.
We validate our approach in a comprehensive in-person user study using Assetto Corsa---a high-fidelity car racing simulator with black-box dynamics---to assess robustness in ``driving on the edge'' scenarios.
We compare both trajectory data and drivers’ perceptions of our \gls{HCSF} assistance against unassisted driving and a conventional safety filter.
Experimental results show that 1) compared to having no assistance, our \gls{HCSF} improves both safety and user satisfaction without compromising human agency or comfort, and 2) relative to a conventional safety filter, our proposed \gls{HCSF} boosts human agency, comfort, and satisfaction while maintaining robustness.
\end{abstract}
\glsresetall
\IEEEpeerreviewmaketitle

\section{Introduction}
\label{sec:intro}

\begin{figure}
    \centering
    \begin{subfigure}{0.49\columnwidth}
        \centering
        \includegraphics[width=\textwidth]{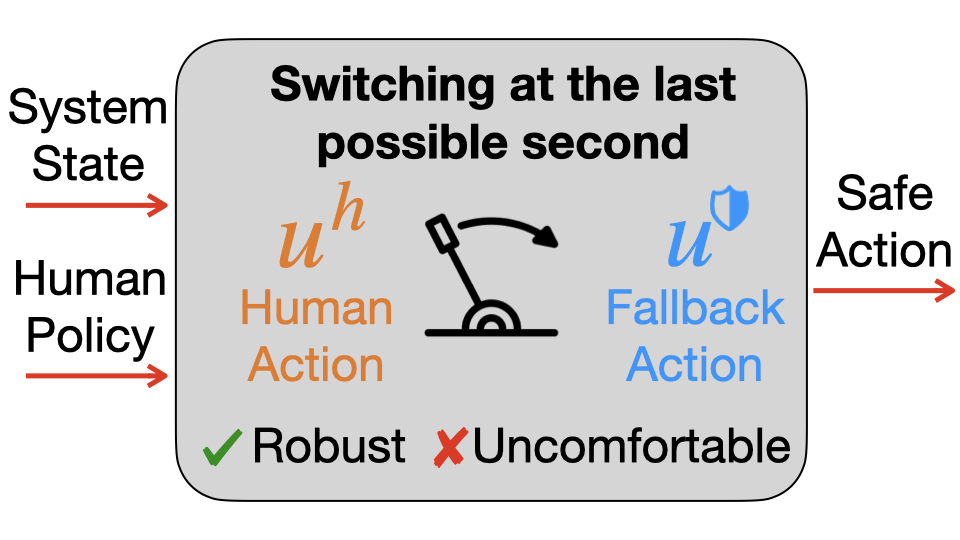}
        \caption{\textit{Last-Resort }Safety Filter}
        \label{fig:saferacing_rss25_figures_half.001}
    \end{subfigure}
        \begin{subfigure}{0.49\columnwidth}
        \centering
        \includegraphics[width=\textwidth]{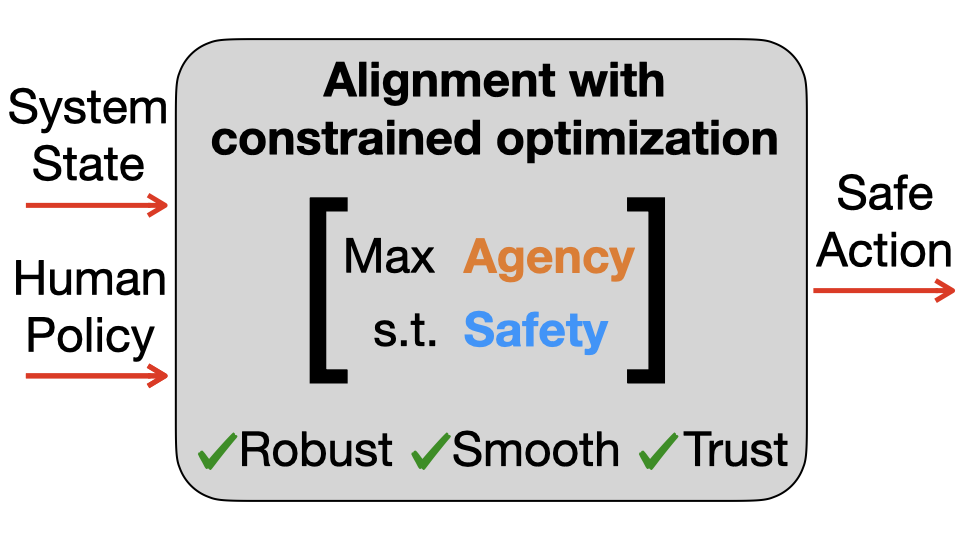}
        \caption{\textit{\princeton{H}uman-\princeton{C}entered} \princeton{S}afety \princeton{F}ilter}
        \label{fig:saferacing_rss25_figures_half.002}
    \end{subfigure}
        \begin{subfigure}{\columnwidth}
        \centering
        \includegraphics[width=\textwidth]{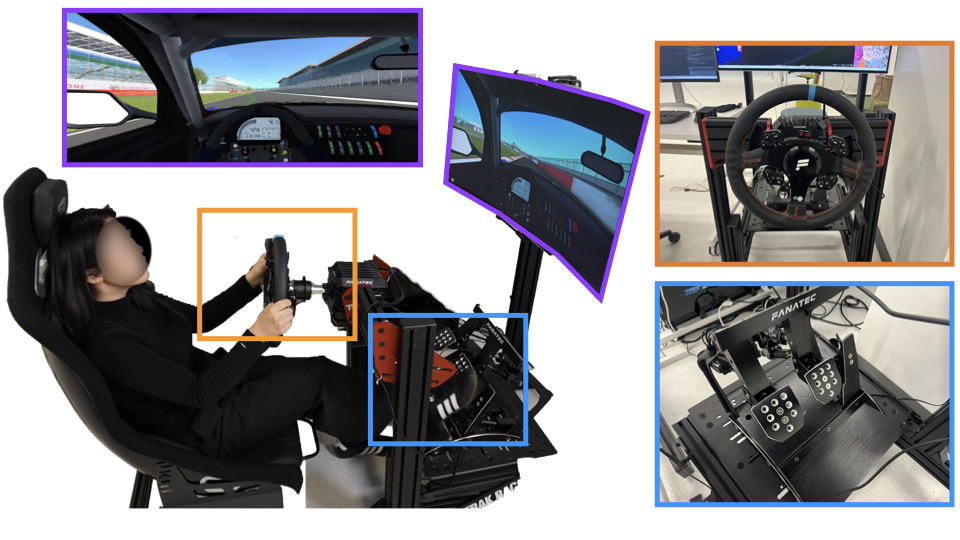}
        \caption{Shared Autonomy for High-Fidelity Car Racing}
        \label{saferacing_rss25_figures_half.003}
    \end{subfigure}
    \caption{Our proposed \gls{HCSF} enables robust and smooth safety interventions for shared autonomy systems. (a) \Gls{LRSF} switches to the best-effort fallback policy at the last possible moment. However, this switching can feel abrupt and uncomfortable for human operators. (b) Our \gls{HCSF} instead intervenes smoothly while promoting human agency, thereby reducing automation surprise and enhancing user experience. (c) Users interact with a high-fidelity racing simulator via a steering wheel and set of pedals (throttle and brake).}
    \label{fig:teaser}
\end{figure}

Recent developments in robot safety provide an exciting opportunity for enhancing human safety and performance in high-stakes situations.
However, augmenting human decision-making with \gls{AI} in a trustworthy way remains an open problem.
A human--AI team in a performance car racing~\cite{wurman2022outracing,decastro2024dreaming,gopinath2024computational} is a representative, yet challenging example, as it pushes safety to the limit.
How should the AI co-pilot assist the human without dulling the driver’s competitive edge?
Should the AI discourage a human from attempting a risky overtaking maneuver on a sharp turn?
When an AI system assists humans in such safety-critical and time-sensitive tasks, maintaining \textit{human agency} is critical.
We need to ensure human awareness of the AI system’s current intent and operating mode, thus avoiding the notorious and sometimes fatal ``automation surprise''~\cite{sarter1997team,jamieson2022b737}.

Safety filters~\cite{bastani2021safe, hsu2023safety} have become an effective approach to ensure safety under an \gls{ODD}, \ie, a clearly defined set of operating conditions for robots to work properly and safely~\cite{sae3259,RR-2662}. Safety filters have been deployed on a wide range of autonomous systems, such as automated vehicles~\cite{zeng2021safety,tearle2021predictive}, legged robots~\cite{hsu2015control,agrawal2017discrete,nguyen2024gameplay}, and aerial navigation~\cite{singletary2022onboard,chen2021fastrack}.
Traditional model-based numerical approaches for safety filter synthesis~\cite{mitchell2005time} result in safety guarantees by design, but they are unable to scale up due to the ``curse of dimensionality''~\cite{bansal2017hamilton}.
Recent research focuses on neural approximation of safety filters~\cite{fisac2019bridging,robey2020learning,bansal2021deepreach,hsu2023isaacs,wang2024magics} that can scale to tens~\cite{nguyen2024gameplay,he2024agile} and even hundreds of state variables~\cite{hu2023deception}. 
While existing safety filters effectively maintain safety, their use in human–AI shared autonomy can result in abrupt, discontinuous interventions that disregard the human operator’s intentions. This undermines the driver’s sense of being in control, creating an uncomfortable and unenjoyable experience. Moreover, such unpredictable, non-transparent behavior can erode confidence and trust in the AI assistant, ultimately degrading team performance and causing the human driver to lose their strategic edge.

\p{Contributions}
To overcome these limitations, we propose a
novel \gls{HCSF} (Fig.~\ref{fig:teaser}) that advances the state of the art in learning-based safety filtering while actively promoting human agency in shared autonomy settings.
In particular, we make three key contributions:
\begin{itemize}
    \item 
    We introduce, to the best of our knowledge, the first \emph{fully model-free \gls{CBF} safety filter}. We learn a neural safety value function through interactions with a black-box system and, at deployment, enforce a safety constraint based on a novel \emph{\gls{Q-CBF}} without any knowledge of system dynamics (\eg, control affine model). Both the synthesis and deployment of our \gls{Q-CBF} are scalable to high-dimensional systems and do not require any knowledge of their dynamics. 

    \item We build upon the learned \gls{Q-CBF} and demonstrate our \gls{HCSF} in \gls{AC}, a high-fidelity racing simulator with black-box dynamics, where the filter is pushed to the limit against all potential failure modes by real human drivers with diverse skill levels. To the best of our knowledge, this is the first time a safety filter has been synthesized, deployed, and evaluated in such a high-dimensional, dynamic shared autonomy setting involving human operators.

    \item We conduct an extensive in-person user study with 83 human participants and conclude, with statistical significance in both trajectory data and human driver responses, that our \gls{HCSF} considerably improves safety and user satisfaction without compromising human agency or comfort relative to having no safety filter. Furthermore, when compared to a conventional safety filter, our \gls{HCSF} offers significant gains in human agency, comfort, and overall satisfaction while maintaining at least the same level of robustness---if not exceeding it.
    
\end{itemize}

\p{Overview}
We organize this paper as follows.
\autoref{sec:related_work} reviews related works, while \autoref{sec:prelim} introduces the problem formulation.
In \autoref{sec:HCSF}, we present our \gls{HCSF} design, emphasizing its key properties and synthesis, and then discuss its practical implementation in \autoref{sec:training}.
\autoref{sec:results} provides our experimental results.
Finally, \autoref{sec:limitation} addresses the limitations and outlines possible future directions, and \autoref{sec:conclusion} concludes the paper.

\section{Related Work}
\label{sec:related_work}

Our work relates to, and builds on, recent advances in human-interactive safety filters and AI-assisted motorsports.

\subsection{Human-Interactive Safety Filters}
A safety filter is a supervisory control scheme that continuously monitors the operation of an autonomous system and intervenes, when necessary, by adjusting its planned actions to prevent potential catastrophic failures.
Safety filters have been increasingly used in high-stakes autonomy applications, ranging from autonomous driving~\cite{leung2020infusing,tearle2021predictive,hu2022sharp,brat2023autonomy}, to aerial navigation~\cite{chen2018hamilton,squires2018constructive,robey2020learning,chen2021fastrack, oh2023safety}, and to legged locomotion~\cite{hsu2015control,he2024agile,nguyen2024gameplay,wang2024magics}.
Recent work by Hsu et al.~\cite{hsu2023safety} provides a unified analysis framework for various safety filters, including \gls{HJ} reachability~\cite{mitchell2005time,bansal2017hamilton,chen2021fastrack}, control barrier functions~\cite{ames2016control,robey2020learning,lindemann2021learning}, model predictive control~\cite{li2020safe,wabersich2021predictive}, and Lyapunov methods~\cite{chow2018lyapunov}.
In general, synthesis of safety filters can be computationally challenging, especially for systems with high-dimensional state space and complex dynamics.
Deep learning has proven to effectively scale up the computation of safety controllers~\cite{fisac2019bridging,hsu2021safety,hu2023deception,nguyen2024gameplay,wang2024magics}.
More recently, methods have been developed that treat these learned neural controllers as an untrusted fallback within a safety filter framework, and robust safety guarantees can be subsequently obtained through runtime verification algorithms such as convex optimization~\cite{hu2020reach,everett2021reachability,gates2023scalable}, forward reachable sets rollouts~\cite{hsu2023isaacs}, and conformal prediction~\cite{lin2024verification}.

When robots are deployed around humans, ensuring safety is paramount to enable their trustworthy integration into people's everyday lives.
However, enforcing safety becomes particularly challenging in human-interactive settings due to coupled motion, limited communication, and potentially conflicting objectives between robots and their human peers.
Early attempts at safe human--robot interaction focus on achieving robust safety by safeguarding against worst-case human decisions~\cite{fisac2018probabilistically,leung2020infusing,tian2022safety}, which may lead to overly conservative robot behaviors~\cite{trautman2010unfreezing}.
Recent research effort has been devoted to designing safety filters that adapt to human decision-making, in hope of improving the robot's task performance without compromising safety.
One popular approach is filter-aware motion planning, which incorporates predictions of the safety filter's behaviors into the robot's task policy~\cite{leung2020infusing,hu2024active}.
This strategy allows the robot to avoid abrupt safety overrides by preempting future costly interventions triggered by unlikely human actions.
Another line of research aims at reducing conservativeness by dynamically adjusting the safety filter's \gls{ODD} according to the robot's evolving uncertainty about the human~\cite{bajcsy2021analyzing,hu2023deception,pandya2024robots}.

While existing human-interactive safety filters enable robots to interact safely and efficiently with \emph{other} humans, similar formulations in human–robot \emph{shared control} settings remain scarce. Recently, research efforts have focused on preserving human agency while enhancing safety in shared autonomy~\cite{broad2019highly, schaff2020residual, reddy2018shared}. However, these approaches do not define a clear \gls{ODD} and lack the principled safety analysis that a safety filter provides, often leading to elevated failure rates. Moreover, some methods rely on knowing the human operator’s policy \emph{a priori}, limiting their robustness when working with groups of human operators who have diverse intentions and skill levels \cite{broad2019highly, reddy2018shared}.

In this work, we draw inspiration from filter‐aware planning to design a safety filter that minimally modifies human actions. Our proposed \gls{HCSF} preserves the principled safety analysis inherited from safety filter theory---in particular, from \gls{HJ} reachability and control barrier functions---while avoiding any explicit model representation of human intentions.

\subsection{AI-Enabled Motorsports}

While modern AI systems surpass human intelligence in competitive sports~\cite{mnih2015human,kaufmann2023champion}, their potential to \textit{augment} human decision-making is underexplored.
High-speed performance car racing presents a domain where safety and seamless collaboration are required to enable a competitive human--AI team---the AI co-pilot must assist the human without dulling the driver’s competitive edge.

Wurman et al.~\cite{wurman2022outracing} demonstrate for the first time that a well-trained neural policy can win a head-to-head competition against some of the world’s
best drivers in a car racing game.
Follow-up works further improve the AI competitiveness via reasoning strategic interactions with data-driven modeling of opponent behaviors~\cite{chen2023learn} and blending model-based dynamic game strategy with data-driven prior knowledge~\cite{lidard2024blending,hu2024think}.
Comparing to fully automated AI motorsports, human--AI collaborative car racing is an emerging, yet relatively underexplored research area.
Gopinath et al.~\cite{gopinath2024computational} propose a multi-task imitation learning approach that enables an automated coaching system that interacts with the student similar to a human teacher.
DeCastro et al.~\cite{decastro2024dreaming} enhance the performance of human--AI teams in car racing by learning a policy that infers and aligns with human intents leveraging a world model.
While AI agents in motorsports have shown promising performance, ensuring the safety of human drivers remains largely unaddressed in a principled manner.
Chen et al.~\cite{chen2021safe} present preliminary results on approximate learning-based safety analysis for autonomous racing, but their approach is limited to the single-car, fully automated setting.

This work presents an \gls{HCSF}, a principled safety filter framework that actively promotes human agency, comfort, and satisfaction. We extensively evaluate our \gls{HCSF} in a large-scale user study using \gls{AC}, marking the first time a safety filter has been tested with both quantitative and qualitative measures of human–AI interaction in a high-fidelity, highly dynamic shared autonomy setting.

\begin{figure*}
    \centering
    \includegraphics[width=\linewidth]{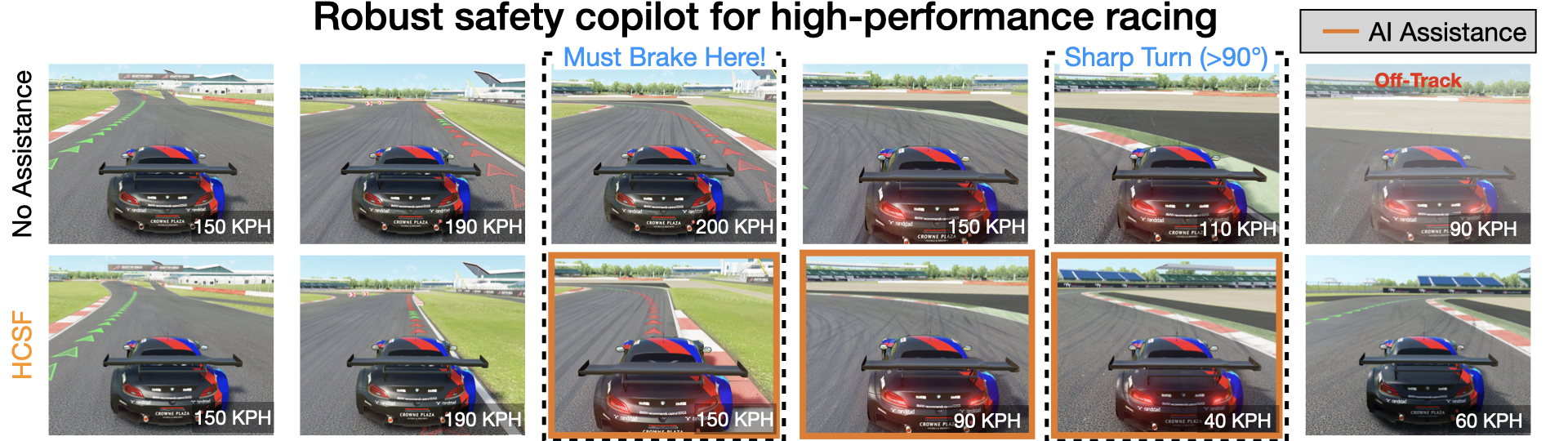}
    \caption{Illustration of our \gls{HCSF} intervention at a hairpin corner (\ie, a sharp turn requiring rapid deceleration). Without safety filter assistance, inexperienced human drivers often miss the braking point, leading to understeering and the vehicle leaving the track. In contrast, our \gls{HCSF} monitors the state and the human action to determine the braking point and provides necessary steering and braking interventions that keep the vehicle on the track. Braking assistance is visible through the rear lights.}
    \label{fig:rollout}
\end{figure*}

\section{Preliminaries and Problem Formulation}
\label{sec:prelim}

We seek to ensure the safe operation of a robot with discrete-time nonlinear dynamics:
\begin{equation}
\label{eq.system_dynamics}
\state_{t+1}=\dyn(\state_t, \ctrl_t),
\end{equation}
where $\state_t\in\xset\subset\reals^{n_x}$ and $\ctrl_t\in\cset\subset\reals^{n_u}$ denote the state and control input at time step $t\in\mathbb{N}$.
The robot's control typically comes from a \emph{task policy} $\policyTask: \xset \rightarrow \cset$.
We define the \emph{failure set} $\failureset$ with a Lipschitz continuous \emph{safety margin function} $g:\xset\rightarrow\reals$:
\begin{equation}
\label{eq.failure_set_def}
\failureset\coloneq\{\state\in\xset\mid g(\state)<0\}.
\end{equation}
States inside $\failureset$ are considered to have already failed in terms of safety. In the context of racing, states that correspond to the race car being outside the track boundaries or in contact with another vehicle should be inside $\failureset$.
The control set $\cset$ and failure set $\failureset$ are core components of the robot's \gls{ODD} (\autoref{subsec:environment}).
The \gls{ODD} may be understood as a social contract that bridges the robot operator, the public, and the policymakers---it provides a clear-cut set of conditions under which the robot is required to operate safely.
To ensure safe robot operation under an \gls{ODD}, we consider a supervisory control framework called safety filters.

\subsection{Safety Filters}
A \emph{safety filter} \cite{hsu2023safety} is
an automated process that continuously monitors the system and intervenes, if deemed necessary, by modifying a candidate action given by the task policy $\policyTask$ to prevent a potentially catastrophic safety failure \textit{in the future}.
Specifically, instead of directly applying the task action $\ctrl_t = \policyTask(\state_t)$, the robot uses an action based on safety filtering:
\begin{equation}
\label{eq:sf}
\ctrl_t = \safetyFilter(\state_t, \policyTask).
\end{equation}
The specific function form of $\safetyFilter$ depends on the intervention type of a safety filter, which includes, \eg, switching, transition, and optimization~\cite[Sec.~3]{hsu2023safety}.
The safety filter only prevents the use of a task action that would compromise future safety, allowing the robot to maintain safety without needing to modify its entire behavior.
In this paper, we use \gls{HJ} reachability analysis to synthesize safety filter $\safetyFilter$~\cite{bansal2017hamilton, mitchell2005time}.

\subsection{Hamilton-Jacobi Reachability Analysis}

We aim to design a safety filter which, given ODD elements \(\failureset\) and \(\cset\), keeps the robot within the \emph{maximal safe set} \(\safeset^* \subset \failureset^\compl \subset \xset\), where $\failureset^\compl=\xset-\failureset$. This set \(\safeset^*\) consists of all states from which there \emph{exists} a control policy that indefinitely prevents the robot from entering \(\failureset\). In theory, \(\safeset^*\) can be computed using Hamilton-Jacobi (HJ) reachability analysis, which employs level-set methods to recast the safety filter synthesis problem as an optimal control problem. Its solution follows from solving the dynamic programming \emph{safety Bellman equation}~\cite{mitchell2005time}:

\begin{equation}
\label{eq.safety_bellman_eq}
\valfunc(\state)=\min\{g(\state), \max_{\ctrl\in\cset}\valfunc(\dyn(\state, \ctrl))\},
\end{equation}
which admits the \emph{safety value function} \(V:\xset\rightarrow\reals\) as its fixed-point solution. Given \(\valfunc(\cdot)\), the maximal safe set \(\safeset^*\) is then defined as:
\begin{equation}
\label{eq.safe_set_def}
\safeset^*\coloneq\{\state \in\xset\mid \valfunc(\state)\geq0\}\subset\failureset^\compl.
\end{equation}
For subsequent extension to our proposed \gls{HCSF}, we adopt the $\qfunc$-function \cite{watkins1992q}---a notion widely used in \gls{RL}---to modify \eqref{eq.safety_bellman_eq} into the \emph{state--action safety Bellman equation}:
\begin{equation}
\label{eq.state_action_safety_bellman_eq}
\qfunc(\state, \ctrl)=\min\{g(\state), \max_{\ctrl^\prime\in\cset}\qfunc(\dyn(\state, \ctrl), \ctrl^\prime)\},
\end{equation}
which admits the \emph{state--action safety value function} \(\qfunc : \xset \times \cset \to \reals\) as its fixed-point solution. This formulation remains equivalent to \eqref{eq.safety_bellman_eq} in the sense that $\valfunc(\state)=\max_{\ctrl\in\cset}\qfunc(\state, \ctrl)$.

We now introduce the \emph{\gls{LRSF}}, a value-based safety filter constructed upon the safety value functions \(\qfunc(\cdot,\cdot)\) and \(\valfunc(\cdot)\):
\begin{equation}
\label{eq.LRSF_def}
\ctrl(\state) =
\begin{cases} 
\policyTask(\state), & \forall \state \in \xset\quad\text{s.t.}\quad \valfunc(\state) > 0 \\
\fallback(\state), & \text{otherwise},
\end{cases}
\end{equation}
where the \emph{safe fallback policy} is defined as
\(\fallback(\state)\coloneq \argmax_{\ctrl \,\in\, \cset}\qfunc(\state,\ctrl)\), 
and \(\policyTask(\cdot)\) is any task-oriented policy that does not explicitly account for safety. We refer to \eqref{eq.LRSF_def} as a ``last-resort'' strategy since the filter does not intervene until \(\valfunc(\state) = 0\): the critical point at which the system is about to exit \(\safeset^*\). Upon reaching the boundary of \(\safeset^*\), \gls{LRSF} fully overrides the control with \(\fallback(\cdot)\) to enforce safety. 

Prior work \cite{chen2021fastrack, mitchell2005time, bansal2017hamilton, fisac2019bridging, chen2018hamilton} has established \gls{LRSF} as a fundamental framework applicable for all \gls{HJ} reachability analysis-based safety filters. This is due to its straightforward yet effective design for enforcing safety and its ``least-restrictive'' nature, allowing the task-oriented policy full freedom until the system reaches the boundary of the maximal safe set.

However, \gls{LRSF} also has notable drawbacks. First, its \emph{task-agnostic} fallback policy often leads to \emph{discontinuous and jerky interventions} \cite{hsu2023safety}. This issue can become more pronounced in a shared autonomy setting, where human operators might feel surprised and confused by abrupt overrides \cite{sarter1997team, jamieson2022b737}. Even though \gls{LRSF} offers maximum freedom until the system reaches the boundary of \(\safeset^*\), its interventions that do not take into account the human's input risk diminishing the operator’s sense of control. Furthermore, the time and space complexities required to solve the Bellman equation~\eqref{eq.safety_bellman_eq} scale exponentially with the state-space dimension, rendering grid-based dynamic programming infeasible for real-world safe robot control.

Our proposed \gls{HCSF} addresses these limitations by synthesizing an output that minimally deviates from the human operator’s input, thereby enhancing both agency and smoothness while still enforcing safety (\autoref{sec:HCSF}). In addition, in \autoref{sec:training}, we leverage recent advances in safety \gls{RL}~\cite{fisac2019bridging,hsu2021safety,wang2024magics} to approximate the $\qfunc$-function via \gls{RL}, enabling the synthesis of a best-effort fallback policy \(\fallback(\cdot)\) for high-dimensional systems.

\subsection{Discrete-Time Control Barrier Functions}
In this subsection, we introduce the definition and implementation of the \gls{DCBF}, another well-established approach for value-based safety filtering.

\begin{definition}[Discrete-time CBF~\cite{agrawal2017discrete}]
\label{def.DCBF} 
A function $h:\xset\rightarrow\reals$ is a \gls{DCBF} for system \eqref{eq.system_dynamics} if $\safeset = \{\state \in\xset\mid \cbf(\state)\geq0\}\subset\failureset^\compl$ and $\exists\alpha\in(0, 1]$ that satisfies:
\begin{equation}
\label{eq.DCBF_def}
\sup_{\ctrl\in\cset}\Delta \cbf(\state, \ctrl)\geq-\alpha \cbf(\state),\quad\forall \state\in \xset,
\end{equation}
where $\Delta \cbf(\state, \ctrl)\coloneq \cbf(\dyn(\state, \ctrl))-\cbf(\state)$.
\end{definition}

Unlike \gls{LRSF}, which imposes safety through hard overrides, a \gls{DCBF} enables \emph{smooth safety interventions} by solving an optimization problem that finds the safety-enforcing action closest to the task action \(\ctrl^\task(\state)\):
\begin{subequations}
\label{eq.nlp}
\begin{align}\ctrl(\state)=\argminB_{\ctrl\in\cset}&\quad\left \| \ctrl^\task(\state)-u \right \|^2,& \label{eq.nlp_cost}\\
\st&\quad\Delta \cbf(\state, \ctrl)\geq-\alpha \cbf(\state),
\label{eq.DCBF_constraint}
\end{align}
\end{subequations}
where \eqref{eq.DCBF_constraint} is the \emph{\gls{DCBF} constraint}. The trade-off is that a \gls{DCBF} no longer enforces safety within the maximal safe set \(\safeset^*\) in general. Instead, it encodes safety with respect to a (smaller) safe set $\safeset = \{\state \in\xset\mid \cbf(\state)\geq0\} \subseteq \safeset^*$.

An \gls{HJ} safety value function \(\valfunc(\cdot)\) is closely linked to a \gls{DCBF} in the sense that, if \(\valfunc(\cdot)\) is continuously differentiable, it automatically qualifies as a valid \gls{CBF} for the maximal safe set \(\safeset^*\)~\cite[Sec.~3.2]{hsu2023safety}. This insight enables the use of \(\valfunc(\cdot)\) in a smooth \gls{CBF} safety filter rather than the \gls{LRSF} alternative---an approach that underpins our proposed \gls{HCSF}.

\begin{remark}
\label{rem.DCBF_relax} 
\autoref{def.DCBF} could be relaxed such that \eqref{eq.DCBF_def} is required to hold for all $\state_t\in\{\state \in\xset\mid \cbf(\state)\geq0\}$. A control input $\ctrl_t\in\cset$ that satisfies \eqref{eq.DCBF_constraint} for any function $\cbf(\cdot)$ meeting the relaxed \gls{DCBF} definition still renders the 0-superlevel set of $\cbf(\cdot)$ forward invariant. Such relaxation of \eqref{eq.DCBF_def} still guarantees safety, but loses the set attractiveness property for the 0-superlevel set of $\cbf(\cdot)$ \cite{cortez2021robust}.
\end{remark}

\section{Smooth Human-Centered Safety Filter for Shared Autonomy}
\label{sec:HCSF}
In this section, we introduce a model-free, human-centered safety filter methodology that builds on \gls{HJ} reachability analysis and \gls{DCBF}s.
We present our \gls{HCSF} formulation and highlight its differences from existing safety filters.

Conventional \gls{CBF} safety filter methods similar to \eqref{eq.nlp} typically require knowledge of the system's dynamics~\cite{ames2016control, zeng2021safety, agrawal2017discrete, singletary2022onboard, cortez2021robust, oh2023safety}, even when using learned barrier functions~\cite{robey2020learning, lindemann2021learning, dawson2022safe, so2024train, zhang2025gcbf+, xiao2023barriernet}. 
This requirement arises for one or both of the following reasons: 1) either a full-order or a simplified dynamical model of the system is utilized to synthesize \gls{CBF} candidates, and 2) knowledge of the dynamics (\eg, control affine model) is leveraged at runtime to enforce the \gls{CBF} safety constraint within an \gls{OCP} (\eg, quadratic program).

While some works explicitly aimed to build and deploy model-free \gls{CBF} safety filters, they have so far fallen short of being \emph{fully model-free}---relying on some combination of simplified models of the system dynamics~\cite{li2022bridging, molnar2021model, cohen2024safety}, predefined low-level controllers~\cite{molnar2021model, cohen2024safety}, and handcrafted fallback policies (\eg, evading maneuvers)~\cite{squires2021model}.
Additionally, recent efforts in learning a \gls{CBF} for latent state representations have proven to be effective for partially observable systems, but they still require a control affine dynamical model~\cite{kumar2024latentcbf}.
Such reliance on knowledge of the system dynamics and the deployment environment can significantly limit the applicability of \gls{CBF}s in complex, real-world scenarios where the dynamical model is often unknown and should be treated as a black-box.
On the other hand, a model-free algorithm for learning a policy together with a barrier certificate was proposed recently~\cite{yang2023model}, but it cannot be used to build a safety filter because the learned policy must be deployed at all times.
Finally, we acknowledge a preprint reporting concurrent efforts toward a model-free state--action \gls{CBF} safety filter~\cite{he2023state}. However, it addresses a finite-horizon safety problem and enforces safety at runtime like a smooth least restrictive safety filter \cite{borquez2024safety} rather than a \gls{CBF} one, fundamentally differing from our work in both mathematical formulation and enforcement of safety.

To this end, we introduce, to the best of our knowledge, the first CBF safety filter that is \emph{fully model-free}.
We first show that the safety value function $\valfunc(\cdot)$ is itself a valid \gls{DCBF} in the sense of \autoref{def.DCBF}.
Then, using the state--action safety value function $\qfunc(\cdot, \cdot)$ which could be learned scalably through black-box interactions with the system via model-free RL-based \gls{HJ} reachability analysis (\autoref{subsec:neural_synthesis}), we propose a method of enforcing a novel \gls{Q-CBF} safety constraint that does not require any information regarding the dynamics.
While theoretical safety guarantees are contingent on the validity of the learned \gls{CBF} (which may be established through statistical analysis~\cite{lin2024verification} or model-based verification~\cite{robey2020learning,hu2020reach}), this is not the focus of our research. Instead, we show our safety filter achieves an extremely high empirical safe rate and effectively preserves human agency.

We now present the \gls{Q-CBF} formulation.

\begin{figure*}
    \centering
    \includegraphics[width=\linewidth]{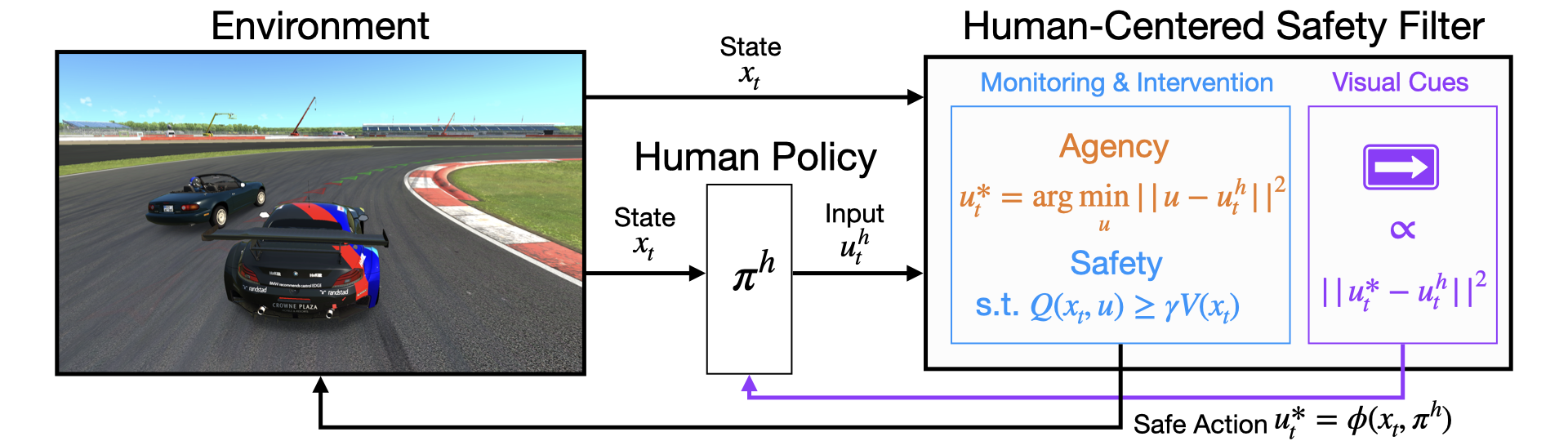}
    \caption{A diagram describing the interaction between a human operator, our proposed \gls{HCSF}, and \gls{AC} game environment. Our \gls{HCSF} utilizes a safety value function that we learn scalably through black-box interactions via model-free \gls{RL}-based \gls{HJ} reachability analysis, and at runtime leverages our novel \gls{Q-CBF} constraint to enforce safety without any knowledge of system dynamics. Moreover, it intervenes minimally and smoothly to enhance human agency and comfort. Finally, our \gls{HCSF} communicates the action modifications to the human driver via visual cues, facilitating transparent human--robot collaborartion.}
    \label{fig:system_control_diagram}
\end{figure*}

\begin{proposition}[\gls{Q-CBF}]
\label{pro:model-free_DCBF}
The safety value function $\valfunc(\state):\xset\rightarrow\reals$, which is a fixed-point solution of the safety Bellman equation \eqref{eq.safety_bellman_eq}, is a valid \gls{DCBF} as defined in \autoref{def.DCBF} and \autoref{rem.DCBF_relax}. The corresponding \gls{DCBF} constraint is:
\begin{equation}
\label{eq.model_free_DCBF_constraint}
\qfunc(\state, \ctrl)\geq\gamma\valfunc(\state),
\end{equation}
where $\gamma\in[0, 1)$. Following the notation from \autoref{def.DCBF}, $\gamma$ is equivalent to $1-\alpha$.
\end{proposition}

\begin{proof}
    The proof is deferred to Appendix~\ref{app:Proofs}.
\end{proof}

We emphasize the key difference between the original \gls{DCBF} constraint \eqref{eq.DCBF_constraint} and our formulation \eqref{eq.model_free_DCBF_constraint}---the requirement (or lack thereof) of the system dynamics. 
Given $\qfunc(\cdot, \cdot)$ that satisfies \eqref{eq.state_action_safety_bellman_eq}, \eqref{eq.model_free_DCBF_constraint} does not require the system dynamics for its evaluation. This enables its application to safety-critical systems with black-box dynamics, namely a high-fidelity car racing simulator. On the contrary, evaluating the original \gls{DCBF} constraint \eqref{eq.DCBF_constraint} does require \textit{a priori} knowledge of the system dynamics, which prohibits it from being applied to systems with unknown dynamics. We also note that \eqref{eq.model_free_DCBF_constraint} constrains the system such that \(\valfunc(\cdot)\) cannot decay below 0. In other words, it keeps the system within the maximal safe set \(\safeset^*\), in contrast to many handcrafted \gls{DCBF}s that often suffer from conservative safe sets.

We now leverage \autoref{pro:model-free_DCBF} to define our proposed \gls{HCSF}, a safety filter tailored for shared autonomy settings (Fig.~\ref{fig:rollout}). Our \gls{HCSF} solves an \gls{OCP} at each timestep to find a safe action that minimally deviates from the human control $\ctrl^\human(\state)$ while satisfying the \gls{Q-CBF} constraint. 

\begin{definition}[Human-Centered Safety Filter]
\label{def.HCSF}

\begin{subequations}
\label{eq.HCSF}
\begin{align}\ctrl(\state)=\argminB_{\ctrl\in\cset}&\quad\left \| \ctrl^\human(\state)-u \right \|^2,& \label{eq.HCSF_cost}\\
\st&\quad\qfunc(\state, \ctrl)\geq\gamma\valfunc(\state),
\label{eq.HCSF_constraint}
\end{align}
\end{subequations}
where $\gamma\in[0, 1)$ is a design parameter that dictates how quickly the safety value function is allowed to decrease over a single timestep.
\end{definition}

The \gls{Q-CBF} constraint ensures that the safety value function does not decay below the specified threshold $\gamma \valfunc(\state)$ at each timestep. The recursive feasibility of \gls{HCSF} is a direct consequence of $\valfunc(\cdot)$ being a valid \gls{DCBF}, as stated in \autoref{pro:model-free_DCBF}. This is formalized in the following proposition:

\begin{proposition}[Recursive Feasibility of HCSF] \label{pro:recursive_feasibility}
The optimization problem in Eq.~\eqref{eq.HCSF} is recursively feasible for $\forall\gamma \in [0, 1)$, given any initial state $\state \in \safeset^*$.
\end{proposition}

\begin{proof}
    The proof is deferred to Appendix \ref{app:Proofs}.
\end{proof}

Together, \autoref{def.HCSF} and \autoref{pro:recursive_feasibility} establish that our \gls{HCSF} \emph{actively promotes human agency} by selecting an action that remains as close as possible to the human’s intended input, while \emph{ensuring the system never leaves the maximal safe set}. Additional information regarding the practical implementation of our \gls{HCSF} together with our choice of $\gamma$ can be found in Appendix \ref{app:Training}. In the following section, we discuss a series of design choices that enable scalable and efficient implementation of the safety filters, along with details on their deployment in a high-fidelity car racing simulation.

\section{Application of HCSF to High-Speed Racing}
\label{sec:training}

Our \gls{HCSF} design in \autoref{sec:HCSF} assumes the knowledge of the state--action safety value function \(\qfunc(\cdot, \cdot)\). However, directly solving \eqref{eq.state_action_safety_bellman_eq} for high-dimensional systems is intractable---a manifestation of the notorious ``curse of dimensionality.'' In this section, we leverage recent advances in safety \gls{RL} to address this challenge and learn both the \(\qfunc\)-function and the best-effort fallback policy in a high-fidelity racing environment. Specifically, we begin by describing the environment setup and then detail the observation, action, safety margin function, and episode termination conditions. Next, we outline our multi-phase training pipeline, which includes warmup and initialization phases that expedite learning by frequently exposing the agent to ``dangerous'' states. Finally, we discuss how we integrate safety filters with visual cues in an effort to enhance the AI’s transparency.

\subsection{Neural Synthesis of Safety Filters}
\label{subsec:neural_synthesis}

Solving the safety Bellman equation~\eqref{eq.safety_bellman_eq} via dynamic programming is intractable for high-dimensional systems, as the computational and memory requirements grow exponentially with the dimensionality of the state space. Even state-of-the-art level-set methods can typically handle at most six continuous state dimensions, rendering them unsuitable for car racing applications.

Recent works \cite{fisac2019bridging, hsu2023isaacs} in safety \gls{RL} address this limitation by proposing a time-discounted variant of \eqref{eq.safety_bellman_eq}, which allows for scalable and effective approximation of the state--action safety value function \(\qfunc(\cdot, \cdot)\) and the best-effort fallback policy \(u^\shield(\cdot)\) via model-free \gls{RL} algorithms, such as \gls{SAC}~\cite{haarnoja2018soft}.

During training, we accumulate a \emph{replay buffer} \(\buffer\) of transitions \(\bigl(x, u, g, x^\prime\bigr)\), where \(g := g(x)\). The critic (\ie, state--action safety value function network) is then trained to predict the future discounted minimum margin by minimizing the Bellman residual, while the actor (\ie, best-effort fallback policy network) is trained to maximize the safety value. Further details on the neural synthesis of safety filters are provided in Appendix~\ref{app:Training}.

\subsection{Human-Machine Interface}

We use Assetto Corsa (\gls{AC}), a high-fidelity racing simulator, together with a gym-compliant interface \cite{remondasimulation}, to facilitate \gls{RL} in a realistic driving environment. Specifically, we adopt the \gls{AC} sim control interface (SCI), which integrates real-time hardware actuation (steering wheel and pedals) with trajectory data from the \gls{AC} game engine and Python implementations of \gls{LRSF} and \gls{HCSF}. A first-person view is displayed on the monitor at 300 Hz. The SCI connects the hardware bus and the host computer via USB to run the safety filter loop (see Fig.~\ref{fig:system_control_diagram}), supporting a 30 Hz control rate---sufficient for a high-fidelity racing environment \cite{remondasimulation, wurman2022outracing}.

The actuation platform consists of a Fanatec CSL DD QR2 wheel base, a Fanatec ClubSport Steering Wheel GT Alcantara V2, and Fanatec Clubsport Pedals V3. To enhance participants’ immersion, we also include a Samsung S39C FHD 75Hz Curved Monitor and a Trak RS6 Racing Simulator rig.

\subsection{Operational Design Domain (ODD)}
\label{subsec:environment}

We select the Silverstone Circuit (GP layout) for both training and deploying the safety filters, as its combination of fast straights and technical corners provides a challenging yet comprehensive proving ground. The ego vehicle is a BMW Z4 GT3, while Mazda MX-5 ND cars serve as opponent vehicles. Since the MX-5 ND is less powerful, it naturally encourages human drivers to attempt overtaking maneuvers. Although a single opponent is deployed during the user study, multiple opponents are used during training to increase on-track interactions and help the ego agent learn effective collision-avoidance strategies. We use 50\% opponent strength and 30\% opponent aggression. The weather condition is set to ``ideal'', track conditions to ``optimum'', temperature to $26^\circ\text{C}$, and wind to $0\,\text{km/h}$. Additionally, traction control, stability control, and ABS are activated, while fuel consumption and tyre wear are turned off.

We learn the state--action safety value function \(\qfunc(\cdot, \cdot)\) and the best-effort fallback policy \(u^\shield(\cdot)\) based on the observation, action, margin function, and episode termination condition detailed below. 

\subsubsection{Observation}
The system operates in a partially observable environment, where each observation is a 133-dimensional vector representing the ego agent’s state and surroundings. This vector includes trajectory data (\eg, speed, angular velocity, tire slip angles, distance to the reference path, and distances to track boundaries computed via ray-casting) from the last four timesteps, as well as the control inputs over the same four timesteps, in order to account for partial observability. The observation vector also contains the look-ahead curvature of the track and information about the nearest opponent, such as relative position, relative velocity, and braking status. A detailed breakdown of this 133-dimensional observation vector is provided in Appendix~\ref{app:Environment}.

\subsubsection{Action}

The normalized action space is defined as $\cset = [-1, 1]^3$, with three continuous channels corresponding to steering, throttle, and brake. Gear changes are handled automatically via the gearbox feature provided by \gls{AC}.

\subsubsection{Margin Function}

The margin function $g : \xset \to \reals$ is defined as the minimum between the signed distance to the track boundary and the signed distance to the nearest opponent, ensuring proximity-based safety constraints for both the environment and opponent vehicles. The corresponding failure set $\failureset$ is defined as in \eqref{eq.failure_set_def}.

\subsubsection{Episode Termination}

If the margin function becomes negative or if the vehicle remains stationary for an elongated time period, the episode terminates and the vehicle is automatically reset to the closest point on the reference path. These episode termination conditions apply to both the neural synthesis of safety filters and the user study.

\subsection{Warmup and Initialization}

In \gls{AC}, resetting the vehicle places it stationary on the closest point of the reference path, making it difficult to gather training data for near-failure scenarios where safety filters are most critical. To address this, we use a two-phase pipeline (warmup and initialization) that accelerates the vehicle to higher speeds under a performance-oriented policy (warmup), then systematically pushes it into more challenging or hazardous states (initialization). By deliberately inducing these ``dangerous'' situations (including adversarial and random maneuvers near the boundary of the safe set), our pipeline ensures the agent encounters a wide range of conditions where the safety filters must intervene effectively. This design both reduces wall-clock training time by avoiding trivial low-speed states and promotes robust learning, as the filter gains experience in precisely the situations where safety intervention is needed most. Full details on the warmup and initialization phases can be found in Appendix~\ref{app:Training}.

\subsection{Training Details}
We train the policy and value networks on a single RTX 4090–equipped machine with an AMD Ryzen 9 7950X 16-core processor. A replay buffer of size 20 million is used, and the actor and critic networks are each updated once per environment step. Both the policy and value networks are three-layer MLPs with 256 neurons per hidden layer, trained with a batch size of 128. The networks are trained for over three weeks (12.8 million environment steps) using the Adam optimizer. We use the same neural approximation of $\qfunc(\cdot, \cdot)$ for all safety filters. Further details on training hyperparameters are provided in Appendix \ref{app:Training}.

\subsection{Visual Cues}

\begin{figure}
    \centering
    \includegraphics[width=\linewidth]{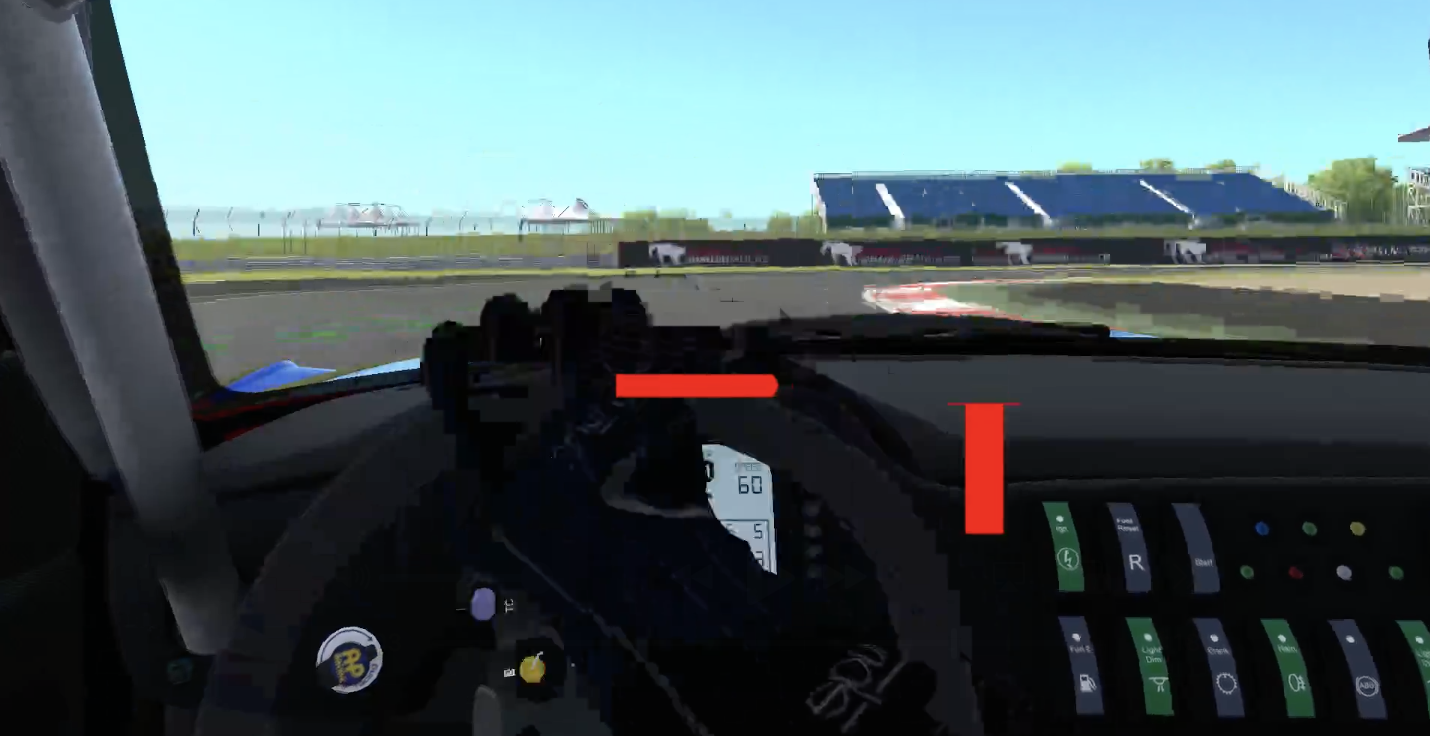}
    \caption{Our HCSF displays two types of visual cues: horizontal arrows that reflect the modifications made to the steering input, and vertical arrows that indicate the corrections made to the throttle/braking inputs. The length of each arrow is proportional to the magnitude of modification made to the corresponding input channel.}
    \label{fig:enter-label}
\end{figure}

In our framework, low-bandwidth visual cues foster transparency and collaboration between the human driver and a safety filter. Whenever an intervention occurs, vertical and horizontal arrows on the screen show both the direction and magnitude of the AI’s corrections relative to the driver’s original input. Specifically, the arrow’s orientation indicates whether the AI is steering more to the left or right compared to the human, or braking more or less compared to the human, while the arrow’s length is proportional to the magnitude of that difference. By mapping each cue to a distinct control channel, we avoid unnecessary information loss. We also omit audio cues, which could add cognitive load or distraction in high-speed scenarios. As illustrated in Fig.~\ref{fig:enter-label}, this setup allows the driver to immediately recognize when and how the system intervenes.

During our user study, all participants were shown a color-coded reference path, which is a series of green and red arrowheads on the track (Fig.~\ref{fig:rollout}). Green arrowheads indicate acceleration, and red arrowheads advise deceleration. This visual aid allows drivers to navigate effectively without prior expertise in sim racing.

\section{Experimental Results}
\label{sec:results}

In this section, we present experimental results that provide evidence for our hypotheses:
\begin{itemize}
    \item \textbf{H1}: Our \gls{HCSF} improves \textbf{safety} and user \textbf{satisfaction} without compromising human \textbf{agency} and \textbf{comfort}, compared to having no filter.
    \item \textbf{H2}: Our \gls{HCSF} improves human \textbf{agency}, \textbf{comfort}, and \textbf{satisfaction} without compromising \textbf{robustness}, compared to \gls{LRSF}.
\end{itemize}

To test these hypotheses, we conducted a large-scale user study in the \gls{AC} simulation environment described in \autoref{subsec:environment}. This study involved 83 participants with diverse driving backgrounds, marking the first time the interaction between human operators and safety filters has been systematically investigated.

\subsection{Baselines} 
\label{subsec:baselines}
We compare our proposed \gls{HCSF} against two baselines---\emph{\gls{LRSF}} and \emph{unassisted driving}---by examining both quantitative trajectory data and qualitative user experience.

For \gls{LRSF}, we follow the formulation in \eqref{eq.LRSF_def}. This filter only intervenes when the system reaches the boundary of the safe set and does not consider the human operator’s input during interventions. We anticipate that such an abrupt approach that disregards human intent may confuse drivers and undermine comfort, thus increasing the risk of automation surprise. In contrast, our \gls{HCSF} solves the optimization problem \eqref{eq.HCSF} to minimally modify the human action while still satisfying the \gls{Q-CBF} constraint, which we hypothesize will enhance human agency, comfort, and overall satisfaction compared to \gls{LRSF}. 
To ensure a fair comparison between our \gls{HCSF} and \gls{LRSF}, we employ the same neural approximation of the state--action safety value function \(\qfunc(\cdot,\cdot)\) for both value-based safety filters. This ensures that \emph{both filters rely on the same safety monitor, while their interventions may differ \eqref{eq.LRSF_def}, \eqref{eq.HCSF}}. In addition, \emph{each filter’s visual cues follow the same proportionality rule}, maintaining consistency in how interventions are conveyed to human drivers.

We also include a control group that receives no safety filter assistance. This allows us to capture any unassisted learning effect---where participants may improve simply through practice---and to gauge the placebo effect of believing one might be assisted by AI, even when no assistance is provided. To control for this placebo effect, all participants are informed that they may receive AI assistance, regardless of their actual assignment. Given that our \gls{HCSF} should substantially reduce accidents, we hypothesize it will achieve superior safety and user satisfaction compared to the unassisted group. Moreover, thanks to its smooth, human-centered interventions, we expect our \gls{HCSF} to preserve human agency and smoothness relative to having no safety filter at all.

\subsection{Metrics} 
\label{subsec:metrics}
We evaluate our core hypotheses using four core metrics: \emph{robustness}, human \emph{agency}, \emph{comfort}, and overall \emph{satisfaction}. In addition, we assess four \emph{filter-specific metrics}---trustworthiness, predictability, interpretability, and competence---to gain further insight into human–safety filter interactions, even though these filter-specific measures are not directly tied to our hypotheses.

\subsubsection{Robustness}
We evaluate the robustness of safety filter–assisted decision making for human drivers using three forms of quantitative trajectory data: out-of-track incidents (per minute), collisions (per minute during close-proximity interaction), and failures (per minute).

Since out-of-track incidents and collisions are two different modes of failure, we normalize them separately. Out-of-track incidents can occur at any moment in a driving session (\eg, a driver might instantly veer off track by sharply flicking the steering wheel), so we divide the total count by the session length. By contrast, collisions can only happen when the ego vehicle is close to an opponent, so we normalize the collision count by the total time spent within a specified distance threshold.

We also gather qualitative data on how robust participants perceive the interaction to be. Even if the quantitative trajectory data indicate strong robustness, drivers may not necessarily feel confident. For example, while many modern vehicles feature lane keeping assist (LKA) systems that effectively reduce lane departures, they can cause a vehicle to bounce between adjacent lane markings. Some drivers lose confidence because of this ``bouncy'' behavior, leading them to disable the feature despite its technical effectiveness \cite{roozendaal2021haptic}. Therefore, we ask participants whether they feel confident in their ability to drive safely throughout each session.

\subsubsection{Agency}
Prior work in cognitive psychology \cite{moore2016sense, braun2018senses, loehr2022sense} and \gls{HRI} \cite{mueller2023toward, prescott2024understanding, collier2025sense} interprets human agency as the correspondence between intended and actual actions. This suggests that agency, in the context of safety filtering, is better framed as a \emph{game of degree} (\ie, how much the system intervenes), rather than a \emph{game of kind} (\ie, whether it intervenes at all). Consistent with this interpretation, we broadly define human agency as the \emph{degree} of control that the driver has over the vehicle. To quantify agency, we define the input modification (I.M.) measure as the $\ell_2$-norm of the difference between the human operator’s raw input and the final input applied after safety filtering:
\begin{equation}
    \text{I.M.}(t) = ||\ctrl_t^\human - \safetyFilter(\state_t, \ctrl_t^\human)||_2,
\end{equation}
where $\safetyFilter$ can be any of the three safety filters in $\{\text{HCSF, LRSF, none}\}$.

In addition to this quantitative trajectory data, we also gather qualitative data on human agency by asking participants how much control they felt they had over the vehicle.

\subsubsection{Comfort}
We measure human comfort using two key forms of quantitative trajectory data: the jerk magnitude and the magnitude of the first-order control input difference.

First, we define the jerk magnitude as:
\begin{equation}
    \text{jerk}(t) = ||\dddot{p}_t||_2,
\end{equation}
where $p_t$ is the vehicle’s position in three-dimensional Euclidean space. Large jerk values can lead to discomfort or motion sickness, whereas smaller jerk values typically indicate a smoother, more comfortable ride \cite{sohn2019driveability, werling2012optimal}.

Next, the squared magnitude of the first-order control input difference (I.D.) is given by:
\begin{equation}
    \text{I.D.}(t) = ||\safetyFilter(\state_t, \ctrl^\human_t)-\safetyFilter(\state_{t-1}, \ctrl^\human_{t-1})||^2_2,
\end{equation}
where smaller I.D. values are associated with better passenger comfort \cite{wang2019path}. In this study, we use I.D. to compare the smoothness of input trajectories obtained from our human-centered optimization problem \eqref{eq.HCSF} with those produced by the best-effort fallback policy \eqref{eq.LRSF_def}.

Finally, we also collect qualitative feedback by asking participants how ``smooth'' they perceived the driving experience to be. We use the term ``smooth'' instead of ``comfortable'' for two reasons: 1) since participants receive no actual motion feedback, it is difficult for them to assess ride comfort, and 2) ``comfortable'' can be mistaken for ``confident,'' given their similarity in everyday usage, which could introduce unwanted ambiguity in the responses.

\subsubsection{Satisfaction}
Because overall satisfaction cannot be measured through quantitative trajectory data, we rely entirely on qualitative data obtained from participants regarding their satisfaction with the overall driving experience. 

\subsubsection{Filter-Specific Metrics}
For additional insights into how human drivers interact with the safety filters, we collect qualitative data on the trustworthiness, predictability, interpretability, and competence of the safety filter assistance. We refer to these four items as ``AI-specific metrics'' because they explicitly address the interventions made by the safety filter, whereas our four core metrics focus on the broader driving experience during a session. For instance, it makes sense to evaluate the smoothness of a session even if no safety filter is deployed, but asking about the trustworthiness of an intervention is only applicable when a filter is actively involved.

\subsection{User Study Design}
\label{subsec:human_study_design}

In this subsection, we describe the design of our user study, including details on participant recruitment, group assignments, and experiment procedure.

\subsubsection{Participant Recruitment} We reached out to potential participants both online and offline. Specifically, we used university-wide mailing lists, social media advertisements, and printed posters to engage faculty, staff, and students at our institution, aiming to minimize biases related to age, gender, or academic major and to ensure a wide range of perspectives. No monetary or academic compensation was provided to participants.

\subsubsection{Group Assignments}
We consider prior driving and video gaming experience to be the most influential factors for our study’s results and explicitly controlled for them. We asked participants to report their driving experience level on a five-point scale, following these criteria: 

\begin{enumerate} 
    \item I have no experience in either real or simulated driving.
    \item I have driving experience but no racing experience.
    \item I have some experience in racing, but only in simulation.
    \item I have extensive racing experience, but only in simulation.
    \item I have real car racing experience.
\end{enumerate}

We assigned participants to each group so that their average initial skill levels remained comparable, ultimately producing group sizes of 25–29 participants. Table~\ref{tab:driving_skill} presents the mean and standard deviation of each group’s reported driving experience. A one-way \gls{ANOVA} followed by Tukey’s \gls{HSD} indicates that p-values for comparisons between any two groups exceed 0.80, suggesting no statistically significant differences in initial skill levels across groups.

\begin{table}[h!]
    \centering
    \renewcommand{\arraystretch}{1.2}
    \begin{tabular}{cccc}
    \hline 
    & HCSF & LRSF  & None \\
    \hline \hline
    Number of Participants & $29$ & $29$ & $25$ \\ \hline
    Average Initial Skill level & $2.17\pm0.54$ & $2.21\pm0.86$ & $2.28\pm0.68$ \\ \hline
    \end{tabular}
    \caption{Number of participants in and the average initial skill level of each group.}
    \label{tab:driving_skill}
\end{table}

\subsubsection{Experiment Procedure} 
In this study, each participant completes three separate driving sessions. During the first session, participants drive for five minutes without any safety filter assistance. This initial session establishes a baseline for each participant’s initial skill level, complementing the group assignment procedure, which also ensures that average initial skill levels remain similar across all groups. All participants are explicitly informed that no assistance is provided in this session. In the second session, participants drive for ten minutes under the assistance of the safety filter corresponding to their assigned group, although they are only told that they may receive AI assistance (with no specification of its type). Finally, as in the first session, participants drive for five minutes without assistance in the third session. This third session is intended to reveal potential over-reliance on safety filters; if participants fully trust the robustness of the filter during the second session, they might rely heavily on it, resulting in marginal or no improvement in their own driving skills. By comparing trajectory data across all three sessions, we can analyze the possibility of over-reliance.

Immediately after each session, participants answer questions regarding the four core metrics: robustness, agency, comfort, and satisfaction. The filter-specific metrics (trustworthiness, predictability, interpretability, and competence) are only queried after the second session, when participants may actually experience safety filter interventions. Each metric is measured using two statements: an affirmative form and a negated form. This ``reverse-coded'' approach helps detect inattentive or biased responses and ensures that the underlying construct is captured from different angles, thus enhancing the reliability of the qualitative measures. We verify that each pair of affirmative and negated statements measures the same construct by conducting a Cronbach’s alpha test, the results of which are reported in Table~\ref{tab:Cronbach_alpha} in the Appendix. All items use a five-point Likert scale. To interpret a participant’s overall response to a given metric, we average the rating from the affirmative statement with $6 - \text{(the rating from the negated statement)}$.

\begin{table}[h!]
    \centering
    \renewcommand{\arraystretch}{1.2}
    \begin{tabular}{ccc}
    \hline
    \textbf{Driving Session} & \textbf{Time} & \textbf{Assistance} \\
    \hline \hline
    Session 1 & 5 minutes & None \\
    \hline
    Session 2 & 10 minutes & LRSF, HCSF, or None\\
    \hline
    Session 3 & 5 minutes & None \\
    \hline
    \end{tabular}
    \caption{Duration of and type of assistance provided in each session.}
    \label{tab:session_time}
\end{table}

\begin{table*}[h!]
    \centering
    
    \renewcommand{\arraystretch}{1.4}
    
    \begin{tabular}{>{\centering\arraybackslash}m{3cm}|m{12cm}}
    \hline
    \textbf{Metric} & \textbf{Related Questions} \\
    \hline\hline
    \textbf{Robustness}
    & \textbf{affirmative}: I felt confident that I could drive safely throughout the race.
    \newline
    \textbf{negated}: I felt nervous about whether I would be able to finish the race without accidents. \\
    \hline
    \textbf{Agency}
    & \textbf{affirmative}: I felt I was in control of the vehicle throughout the race.
    \newline
    \textbf{negated}: There were times when the vehicle wasn't doing what I wanted. \\
    \hline
    \textbf{Comfort}
    & \textbf{affirmative}: The driving experience felt smooth overall.
    \newline
    \textbf{negated}: There were times when the drive felt jerky. \\
    \hline
    \textbf{Satisfaction}
    & \textbf{affirmative}: Overall, I am happy with how the race went.
    \newline
    \textbf{negated}: This race did not go as well as I thought it could have. \\
    \hline
    \textbf{Trustworthiness}
    & \textbf{affirmative}: I trusted the AI Assistant to keep me safe throughout the race.
    \newline
    \textbf{negated}: I felt uneasy about whether the AI Assistant would get me into an accident. \\
    \hline
    \textbf{Predictability}
    & \textbf{affirmative}: The AI Assistant's actions were predictable.
    \newline
    \textbf{negated}: The AI Assistant's actions surprised me at times. \\
    \hline
    \textbf{Interpretability}
    & \textbf{affirmative}: The AI Assistant's actions made sense to me.
    \newline
    \textbf{negated}: Sometimes, I couldn't figure out why the AI Assistant was taking actions. \\
    \hline
    \textbf{Competence}
    & \textbf{affirmative}: The AI Assistant seemed to have a good grasp of the situation.
    \newline
    \textbf{negated}: Sometimes, the AI Assistant didn't seem to know what it was doing. \\
    \hline
    \end{tabular}
    \caption{4 core metrics and 4 filter-specific metrics together with their corresponding affirmative/negated questions.}
    \label{tab:questionnaire}
\end{table*}

\subsection{Results}

In this subsection, we present the results of our user study. We validate our hypotheses on a metric-by-metric basis, using both qualitative and quantitative data for the four core metrics: robustness, agency, comfort, and satisfaction. We also analyze qualitative responses for filter-specific metrics to gain further insight into how participants interact with each safety filter.

When analyzing quantitative or qualitative data that follow a two-factor (session and group) repeated-measures design---i.e., data collected across multiple sessions with an interest in comparing different groups---we employ a Mixed \gls{ANOVA} model to account for both within-subject (session) and between-subject (group) factors. Whenever a significant interaction between session and group emerges in the Mixed \gls{ANOVA} test, we proceed to the \gls{SME} analysis to check how the groups differ separately in session 1 and session 2. Then, if the \gls{SME} analysis results in statistical significance, we perform Tukey's \gls{HSD} test to pinpoint precisely how each group's distribution of data diverges from others within a session, thereby clarifying the role of the safety filters and how they influence each metric.

For data that do not meet this two-factor repeated-measures structure (i.e., data not measured over multiple sessions), we use a one-way \gls{ANOVA} that considers only the between-subject (group) factor. If the \gls{ANOVA} result indicates statistical significance, we then apply Tukey's \gls{HSD} test to compare the distributions pairwise for each group combination.

For all numerical data depicted in bar or box plots, we illustrate statistical significance using asterisks to compare pairs of groups. Formally, under the null hypothesis $H_0$, we assume the data from any two groups come from the same underlying distribution. If a statistical test (\eg, Mixed \gls{ANOVA} followed by \gls{SME} and Tukey's \gls{HSD}) indicates that $p<0.05$, we reject $H_0$ at the 5\% level and mark the pair with one asterisk ($\ast$). Likewise, we use two asterisks ($\ast\ast$) when $p<0.01$, and three asterisks ($\ast\ast\ast$) when $p<0.001$. Thus, more asterisks correspond to stronger evidence against $H_0$ and hence a more statistically significant difference between the groups’ data.

\subsubsection{Robustness}
As shown in Fig.~\ref{fig:safety_metrics}, our \gls{HCSF} maintained near-zero failures across the study, demonstrating strong robustness under diverse human actions. It significantly reduced out-of-track incidents per minute and overall failures per minute compared to the unassisted group. While our \gls{HCSF} also reduced out-of-track incidents and collisions relative to \gls{LRSF}, these differences were not statistically significant.

A similar trend emerges in the qualitative results. As shown in Fig.~\ref{fig:box_s1s2_and_delta_robustness}, participants assisted by our \gls{HCSF} in session 2 felt significantly more confident in the safety filter’s robustness than those in the unassisted group. In addition, the difference in each participant’s confidence score between sessions 1 and 2 was significantly greater for our \gls{HCSF} than for unassisted driving. Although participants in the \gls{HCSF} group also reported higher confidence than those in the \gls{LRSF} group, this difference was not statistically significant.

Hence, we conclude:
\begin{itemize}
    \item Our \gls{HCSF} dramatically improves safety in decision-making compared to having no filter.
    \item Our \gls{HCSF} is at least as robust as  \gls{LRSF}, if not more.
\end{itemize}

Although our \gls{HCSF} did reduce collisions relative to \gls{LRSF} and unassisted driving, there was no statistically significant difference. We conjecture that this may be due to two main causes. First, during the training of the neural approximation of the state--action safety value function and the best-effort fallback policy, we did not treat the opponent as an adversarial agent; instead, we considered the opponent as a part of the environment. We did so because treating the opponent as adversarial would prohibit side-by-side racing---making overtaking impossible---since the opponent could easily induce a collision by side-swiping the ego vehicle. Additionally, \gls{AC} does not allow users to dictate actions for opponent vehicles, thus preventing the setup of an adversarial multi-agent \gls{RL}. In other words, in order for a car race to occur, we cannot guarantee absolute robustness against collisions. Second, participants frequently attempted overtakes in unorthodox parts of the track, where real-world racing discourages such maneuvers. Although they might successfully overtake without colliding, they often carried excessive speed and ended up driving off-track. Consequently, the number of collisions may have been underreported in session 1 and for the unassisted group during session 2.

\begin{figure*}[!h] 
    \centering
    \begin{subfigure}{0.65\columnwidth}
        \centering
        \includegraphics[width=\textwidth, trim=5 20 5 20, clip]{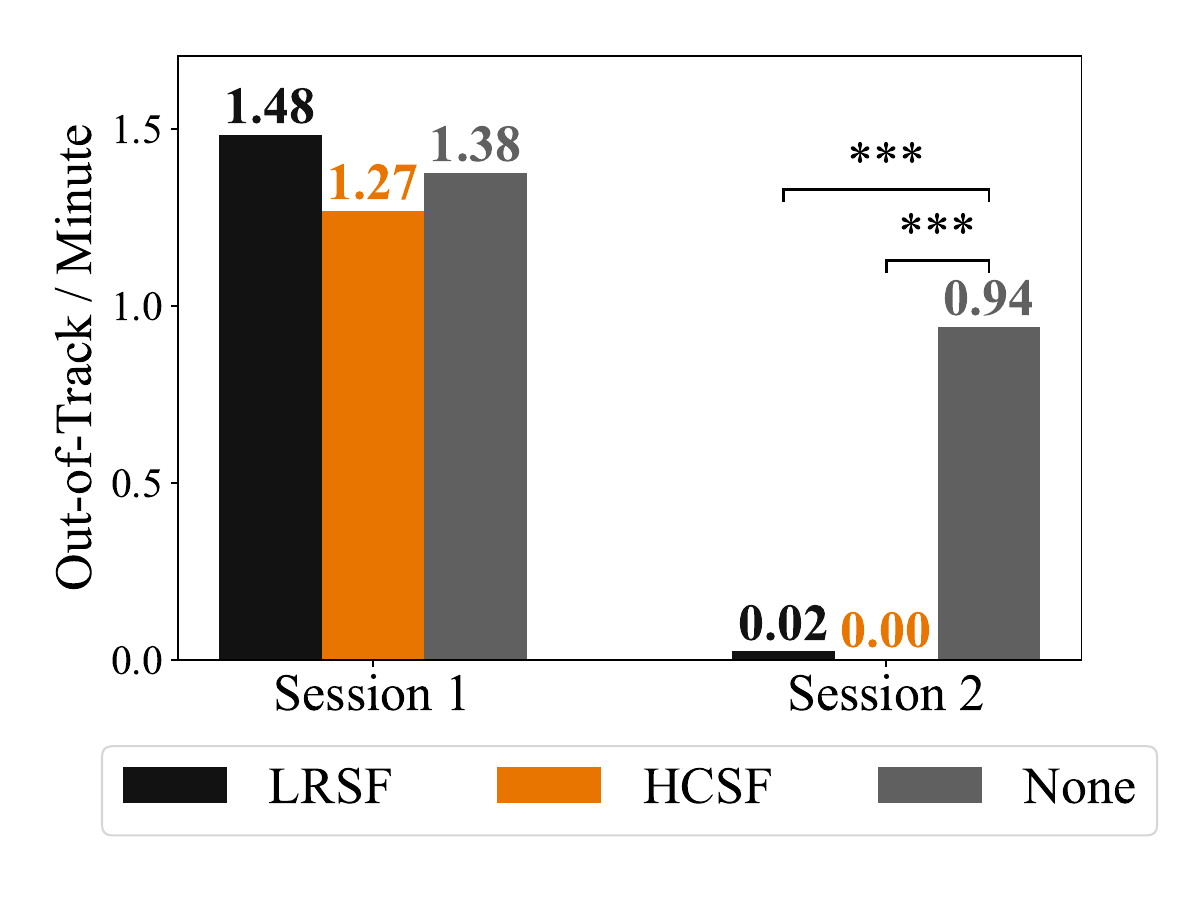}
        \caption{Out-of-track incidents per minute.}
        \label{fig:out_of_track_per_min_side_by_side}
    \end{subfigure}
    \hfill
    \begin{subfigure}{0.65\columnwidth}
        \centering
        \includegraphics[width=\textwidth, trim=5 20 5 20, clip]{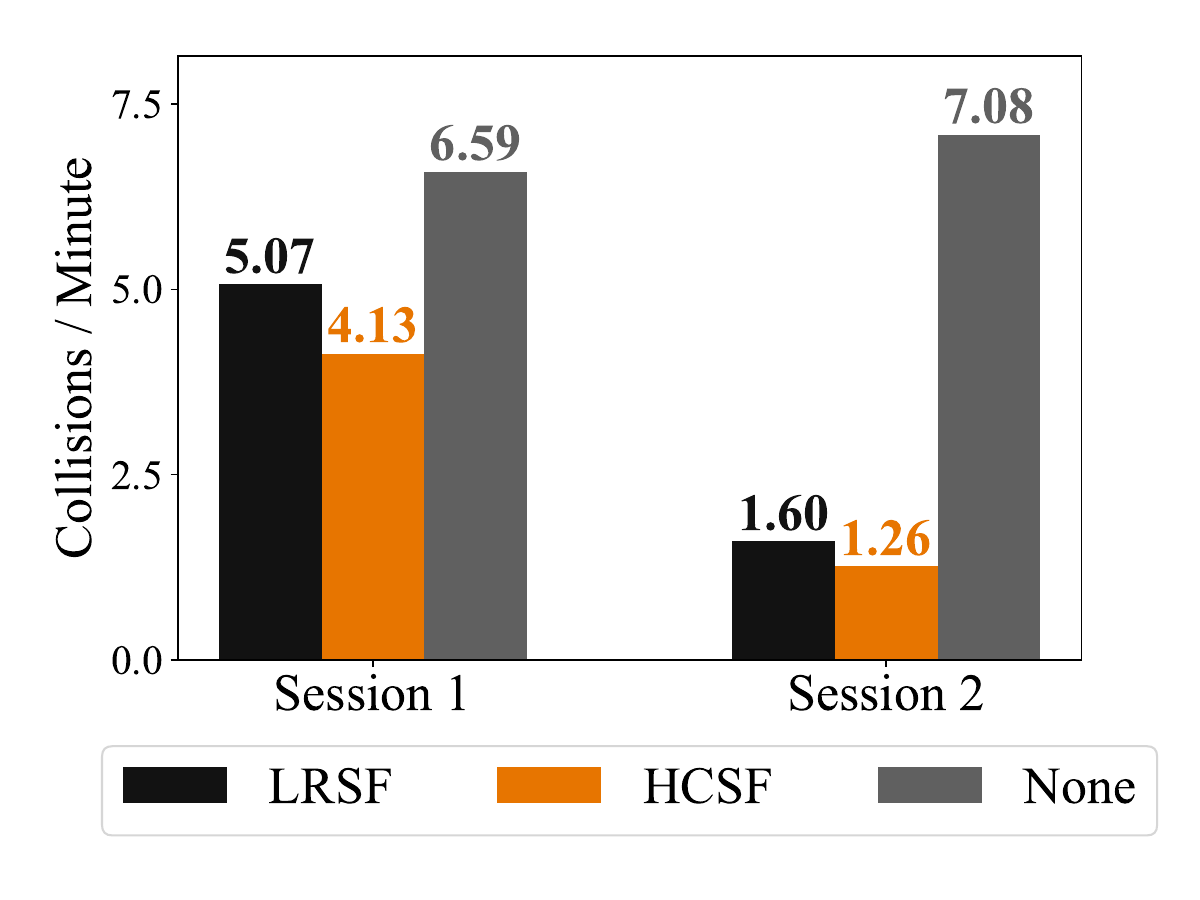}
        \caption{Collisions per minute ($<3m$ distance).}
        \label{fig:collisions_per_prox_side_by_side_3m}
    \end{subfigure}
    \hfill
    \begin{subfigure}{0.65\columnwidth}
        \centering
        \includegraphics[width=\textwidth, trim=5 20 5 20, clip]{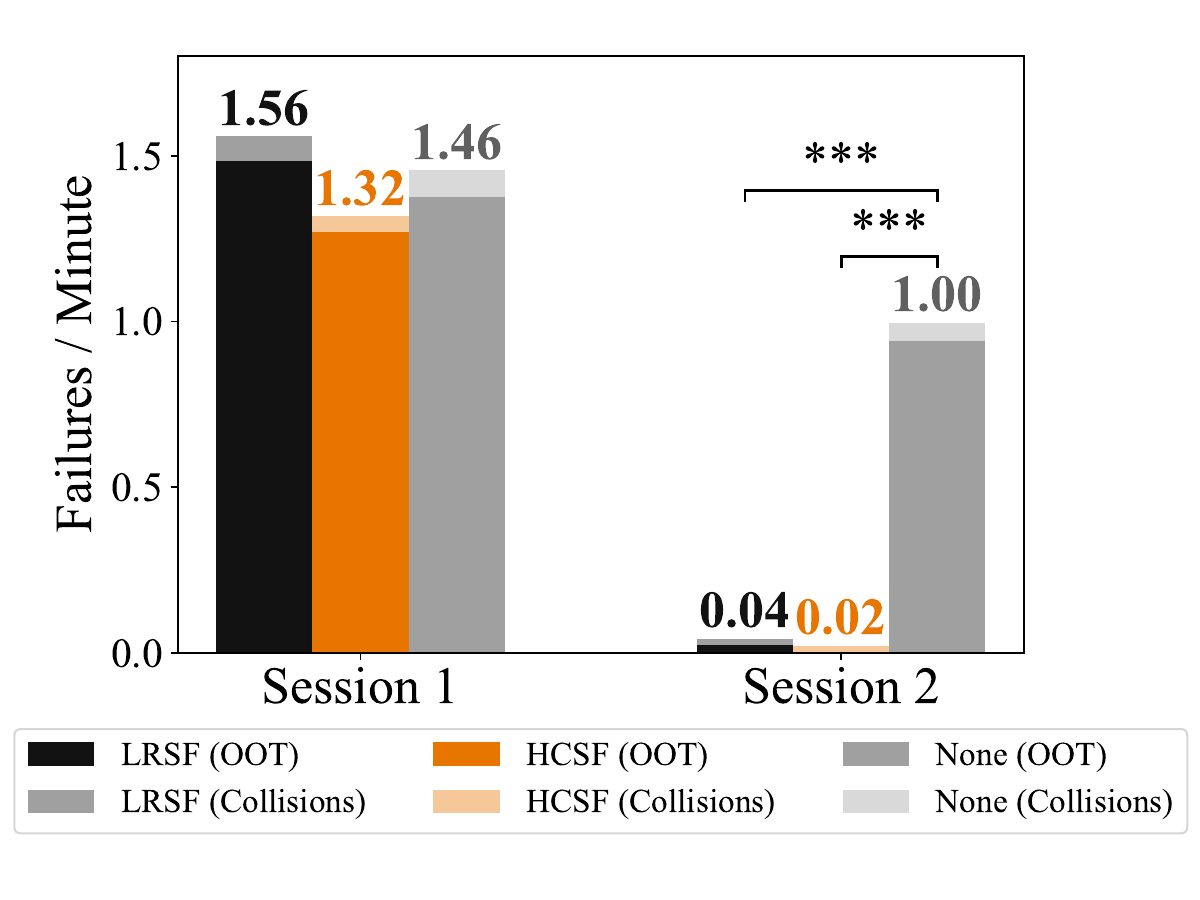}
        \caption{Total number of failures per minute.}
        \label{fig:failures_per_min_stacked}
    \end{subfigure}
    \caption{Our \gls{HCSF} achieved near-zero failures throughout the user study, demonstrating significant enhancement in safety compared to unassisted human driving. Although our \gls{HCSF} outperformed \gls{LRSF} in both failure modes, the differences were not statistically significant. Statistical significance is marked with asterisks, where more asterisks indicate larger significance.}
    \label{fig:safety_metrics}
\end{figure*}

\begin{figure}[!h]
    \centering
    \includegraphics[width=\columnwidth]{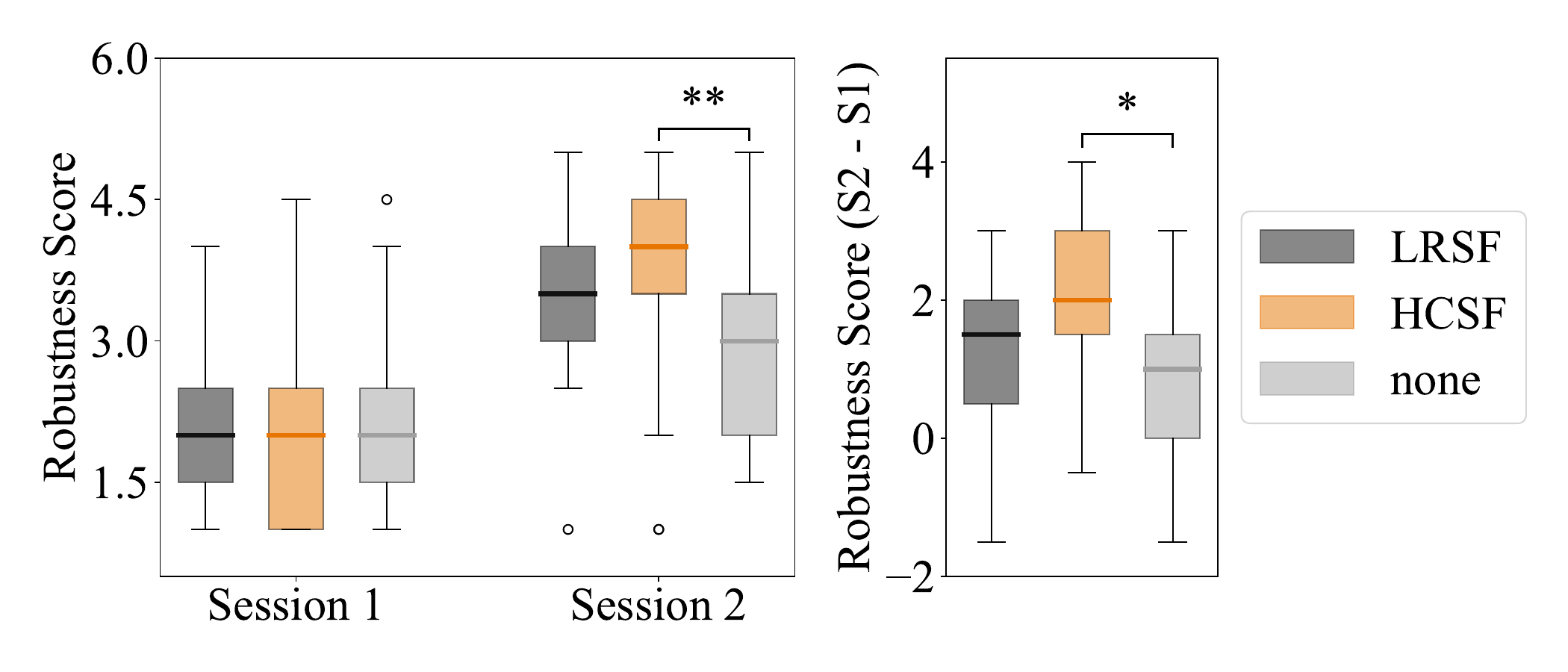}
    \caption{Qualitative measures of robustness across all three groups in both sessions. Participants in the \gls{HCSF} group reported significantly higher confidence in session 2 compared to the unassisted group. Central marks, bottom, and top edges of the boxes indicate the median, 25th, and 75th percentiles, respectively. The maximum whisker length is set to $1.5$ times the standard deviation, which gives $99.7$ percent coverage for normally distributed data. Statistical significance is marked with asterisks, where more asterisks indicate larger significance.}
    \label{fig:box_s1s2_and_delta_robustness} 
\end{figure}

\subsubsection{Agency}
Fig.~\ref{fig:histogram_abs_action_diff_L2_nonzero} reports the \textit{distribution of I.M. among all participants in each group across all timesteps}---including those when no safety filter intervention occurred---for the \gls{HCSF} and \gls{LRSF} groups, thus jointly capturing intervention frequency and magnitude. Notably, our \gls{HCSF} reduces the frequency of interventions with I.M.$>$1.0 relative to \gls{LRSF}, so the filtered actions remain closer to the human operator's intended actions. While our \gls{HCSF} intervenes more frequently than \gls{LRSF} (30.3\% vs. 19.7\% of the driving session duration), its interventions are much smaller in magnitude (0.184 vs. 0.305 in average I.M.).
This result shows that our \gls{HCSF} significantly improves human agency by reducing the input modification magnitude over \gls{LRSF}. By contrast, \gls{LRSF} does not take into account the human action, which can lead to unnecessarily large, abrupt corrections that undermine the driver’s sense of control. This tendency is even more pronounced in the \gls{ECDF} of I.M. in Fig.~\ref{fig:cdf_abs_action_diff_L2_nonzero}, and Fig.~\ref{fig:boxplot_avg_mod_L2_session2_only} shows that the \gls{HCSF} group has a significantly lower average I.M. than the \gls{LRSF} group.

Qualitative measures also strongly support the improved human agency of our \gls{HCSF} over \gls{LRSF}. As shown in Fig.~\ref{fig:box_s1s2_and_delta_agency}, participants assisted by our \gls{HCSF} during session 2 reported a significantly stronger sense of being in control compared to those assisted by \gls{LRSF}. This difference is even more pronounced if we consider the change in each participant’s agency score between sessions 1 and 2. Meanwhile, the unassisted group recorded the highest average agency score overall---unsurprising since they had full control of the vehicle in both sessions---but its difference from the \gls{HCSF} group was not statistically significant in both session 2 scores or the improvement across sessions. Finally, the \gls{LRSF} group exhibited a drop in agency from session 1 to session 2, underscoring how \gls{LRSF} undermines the human operator's sense of control. By contrast, our \gls{HCSF} manages to preserve agency at a level comparable to having no filter at all.

Therefore, we conclude: 
\begin{itemize} 
    \item Our \gls{HCSF} does not compromise human agency compared to having no filter. 
    \item Our \gls{HCSF} significantly enhances human agency compared to \gls{LRSF}. 
\end{itemize}

\begin{figure}[h]
    \centering
    \includegraphics[width=0.75\columnwidth]{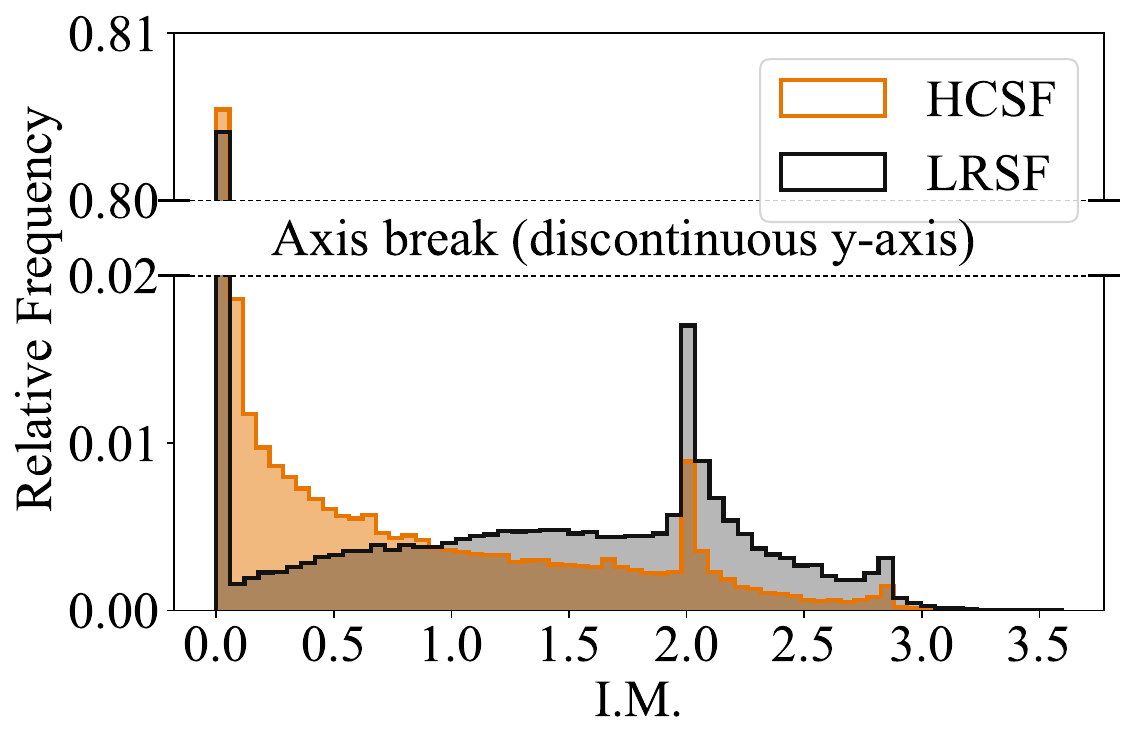}
    \caption{Histogram of I.M. for the \gls{HCSF} and \gls{LRSF} groups over all timesteps, including those when no safety filter intervention occurred. Compared to \gls{LRSF}, which often produces large input modifications that undermine human agency, our \gls{HCSF} reduces the frequency of such large modifications by offering human-centered ``nudges'' of smaller magnitude.}
    \label{fig:histogram_abs_action_diff_L2_nonzero}
\end{figure}

\begin{figure}[!h]
    \centering
    \begin{subfigure}[b]{0.61\columnwidth}
        \centering
        \includegraphics[width=\textwidth]{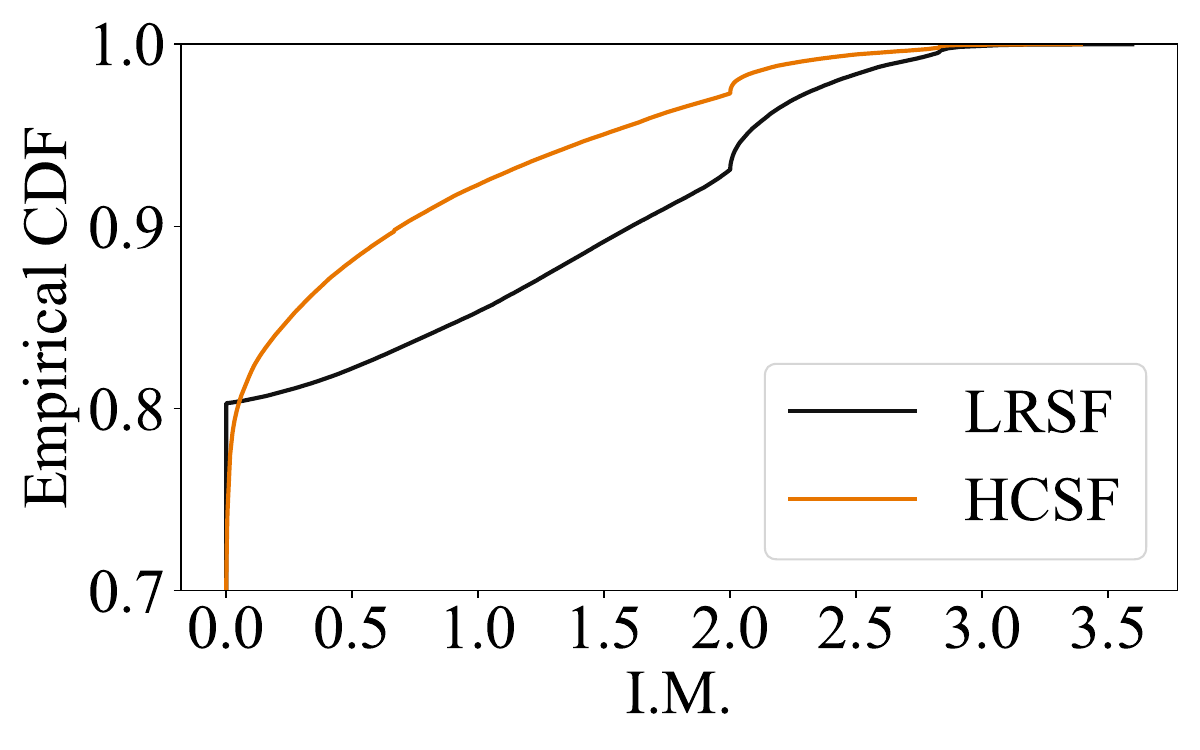}
        \caption{\gls{ECDF} plot of I.M.}
        \label{fig:cdf_abs_action_diff_L2_nonzero}
    \end{subfigure}
    \hfill
    \begin{subfigure}[b]{0.37\columnwidth}
        \centering
        \includegraphics[width=\textwidth]{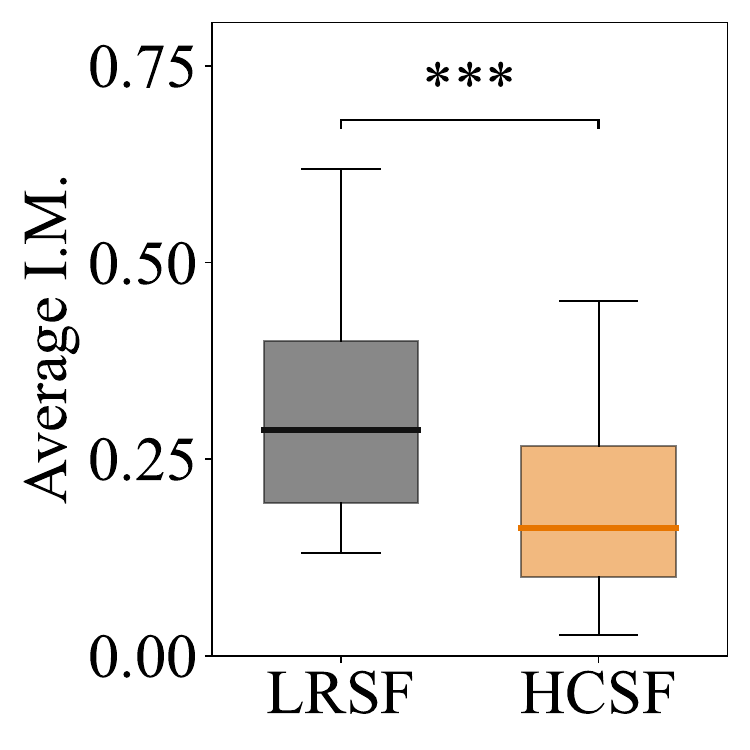}
        \caption{Box plot of average I.M.}
        \label{fig:boxplot_avg_mod_L2_session2_only}
    \end{subfigure}
    \hfill
    \caption{(a) For each group, the presented \gls{ECDF} aggregates the I.M. values from all participants across every timestep, including those when no safety filter intervention occurred. (b) The box plot summarizes the distribution of each participant’s average I.M. within each group. Our \gls{HCSF} yielded a significantly smaller I.M. compared to \gls{LRSF}, suggesting better retention of human agency.}
    \label{fig:agency_cdf_box}
\end{figure}

\begin{figure}[!h]
    \centering
    \includegraphics[width=\columnwidth]{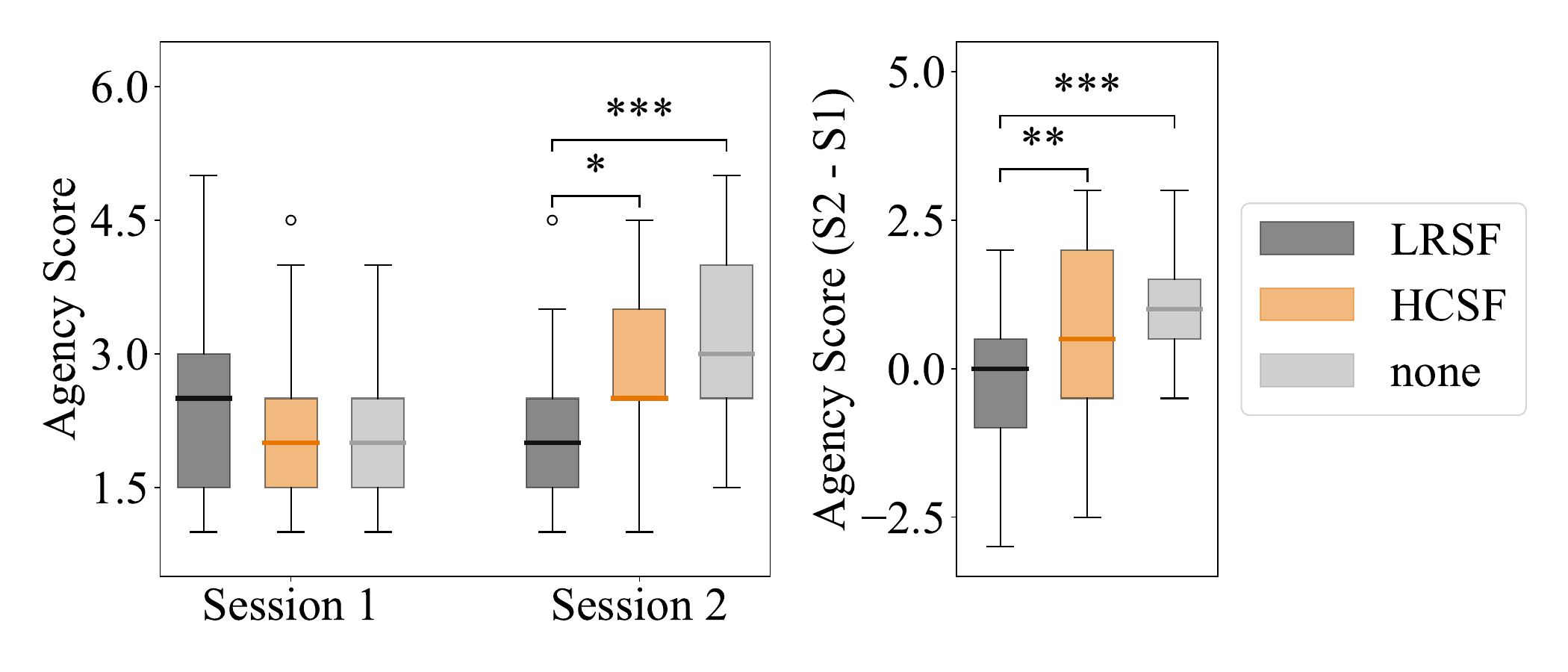}
    \caption{Qualitative measures of human agency across all three groups in both sessions. Participants in the \gls{HCSF} group reported a significantly stronger sense of being in control compared to those in the \gls{LRSF} group, while no significant difference emerged between the \gls{HCSF} and unassisted groups. }
    \label{fig:box_s1s2_and_delta_agency} 
\end{figure}

\subsubsection{Comfort}
Although our \gls{HCSF} does not explicitly address comfort in the optimization problem \eqref{eq.HCSF}, we hypothesize that it will still provide a smoother ride compared to \gls{LRSF}. We expect our \gls{HCSF} to minimize deviations from the human input \eqref{eq.HCSF_cost}, effectively ``anchoring'' the filtered action around the human operator’s commands. Because humans are physically limited in how quickly they can turn the wheel or press the pedals, their inputs naturally form smooth and continuous control trajectories. As a result, the resulting control actions under our \gls{HCSF} assistance are anticipated to inherit some of that smoothness, thus translating into non-jerky movement. In contrast, \gls{LRSF}, which disregards the human input, can cause larger discontinuities in the control signals when switching abruptly between human control and the best-effort fallback policy. Moreover, the best-effort fallback policy itself is a neural network trained with no explicit smoothness reward, further contributing to potentially jerky control.

The \gls{ECDF} plot in Fig.~\ref{fig:AIonly_diff_cdf_s2} illustrates the I.D. distribution among all participants in each group across all timesteps where a safety filter intervened (and hence no data for the unassisted group). We observe our \gls{HCSF} producing a denser distribution at smaller I.D. values. In other words, our \gls{HCSF} tends to yield smoother inputs than \gls{LRSF}. The box plot in Fig.~\ref{fig:avg_1st_diff_box} reinforces this observation by showing a significant difference in average I.D. distribution between the \gls{LRSF} and unassisted groups. Moreover, although the unassisted group exhibits the smallest average I.D., its difference from the \gls{HCSF} group is not statistically significant.

A similar pattern appears in Fig.~\ref{fig:jerk_cdf_s2_AIonly}, which shows the jerk distribution among all participants in each group across all timesteps where a safety filter intervened (and hence no data for the unassisted group). Here, our \gls{HCSF} produces a denser distribution at smaller jerk compared to \gls{LRSF}, indicating smoother, more comfortable driving. The box plot in Fig.~\ref{fig:avg_jerk_box} further confirms that our \gls{HCSF} significantly reduces average jerk relative to \gls{LRSF}, while its difference from unassisted driving is not statistically significant.

Finally, Fig.~\ref{fig:box_s1s2_and_delta_comfort} illustrates the participants' sense of smoothness. Those assigned with \gls{LRSF} reported a decline in their sense of smoothness from session 1 to session 2, indicating discomfort arising from abrupt and discontinuous \gls{LRSF} interventions. In contrast, participants assisted by our \gls{HCSF} reported a significantly higher smoothness score in session 2 compared to those assisted by \gls{LRSF}, and this difference becomes even more pronounced when examining the change in each participant’s score between sessions 1 and 2. Similar to the agency results, the unassisted group exhibited the highest average smoothness score overall, which is unsurprising given the absence of interventions of any sort. Nonetheless, our focus here is on how well our \gls{HCSF} preserves a smooth driving experience; the gap between the \gls{HCSF} and unassisted groups was not statistically significant in either the session 2 scores or in the improvement across sessions.

Therefore, we conclude: 
\begin{itemize} 
    \item Our \gls{HCSF} does not compromise comfort compared to having no filter. 
    \item Our \gls{HCSF} significantly improves comfort compared to \gls{LRSF}. 
\end{itemize}

\begin{figure}
    \centering
    \begin{subfigure}[b]{0.49\columnwidth}
        \centering
        \includegraphics[width=\textwidth]{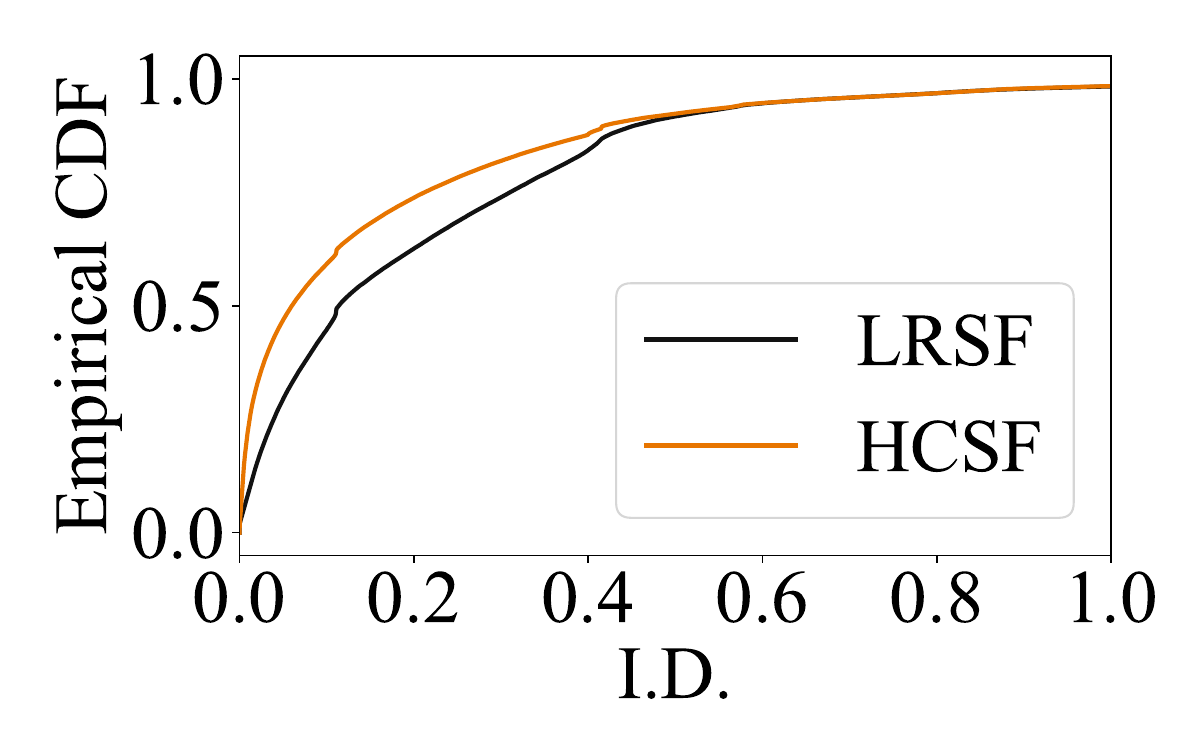}
        \caption{\gls{ECDF} plot of I.D.}
        \label{fig:AIonly_diff_cdf_s2}
    \end{subfigure}
    \begin{subfigure}[b]{0.49\columnwidth}
        \centering
        \includegraphics[width=\textwidth]{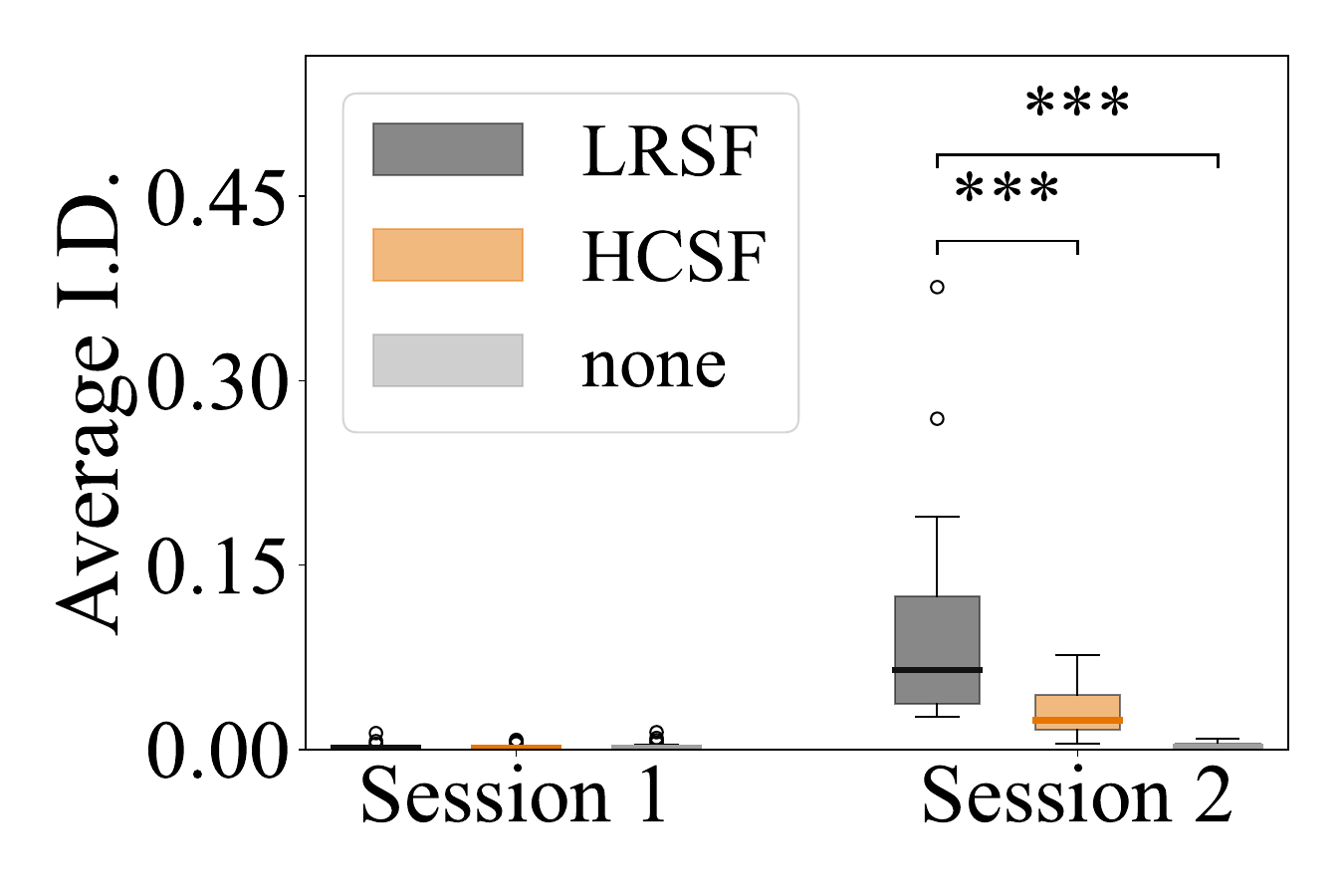}
        \caption{Box plot of average I.D.}
        \label{fig:avg_1st_diff_box}
    \end{subfigure}
    \caption{Our \gls{HCSF} yielded a significantly smaller I.D. compared to \gls{LRSF}, indicating smoother input trajectories, while its difference from unassisted driving was not statistically significant. Central marks, bottom, and top edges of the boxes indicate the median, 25th, and 75th percentiles, respectively. The maximum whisker length is set to $1.5$ times the standard deviation, which gives $99.7$ percent coverage for normally distributed data. Statistical significance is marked with asterisks, where more asterisks indicate larger significance.}
    \label{fig:smoothness_ID}
\end{figure}

\begin{figure}
    \centering
    \begin{subfigure}[b]{0.49\columnwidth}
        \centering
        \includegraphics[width=\textwidth]{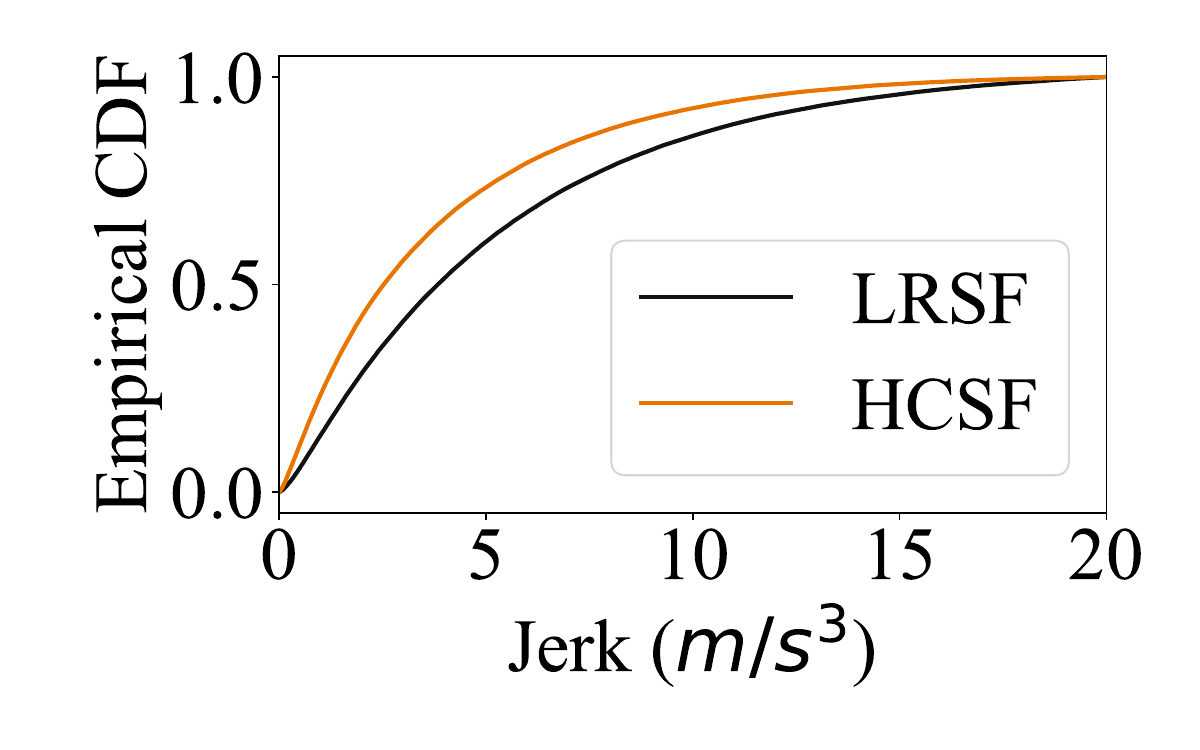}
        \caption{\gls{ECDF} plot of jerk.}
        \label{fig:jerk_cdf_s2_AIonly}
    \end{subfigure}
    \begin{subfigure}[b]{0.49\columnwidth}
        \centering
        \includegraphics[width=\textwidth]{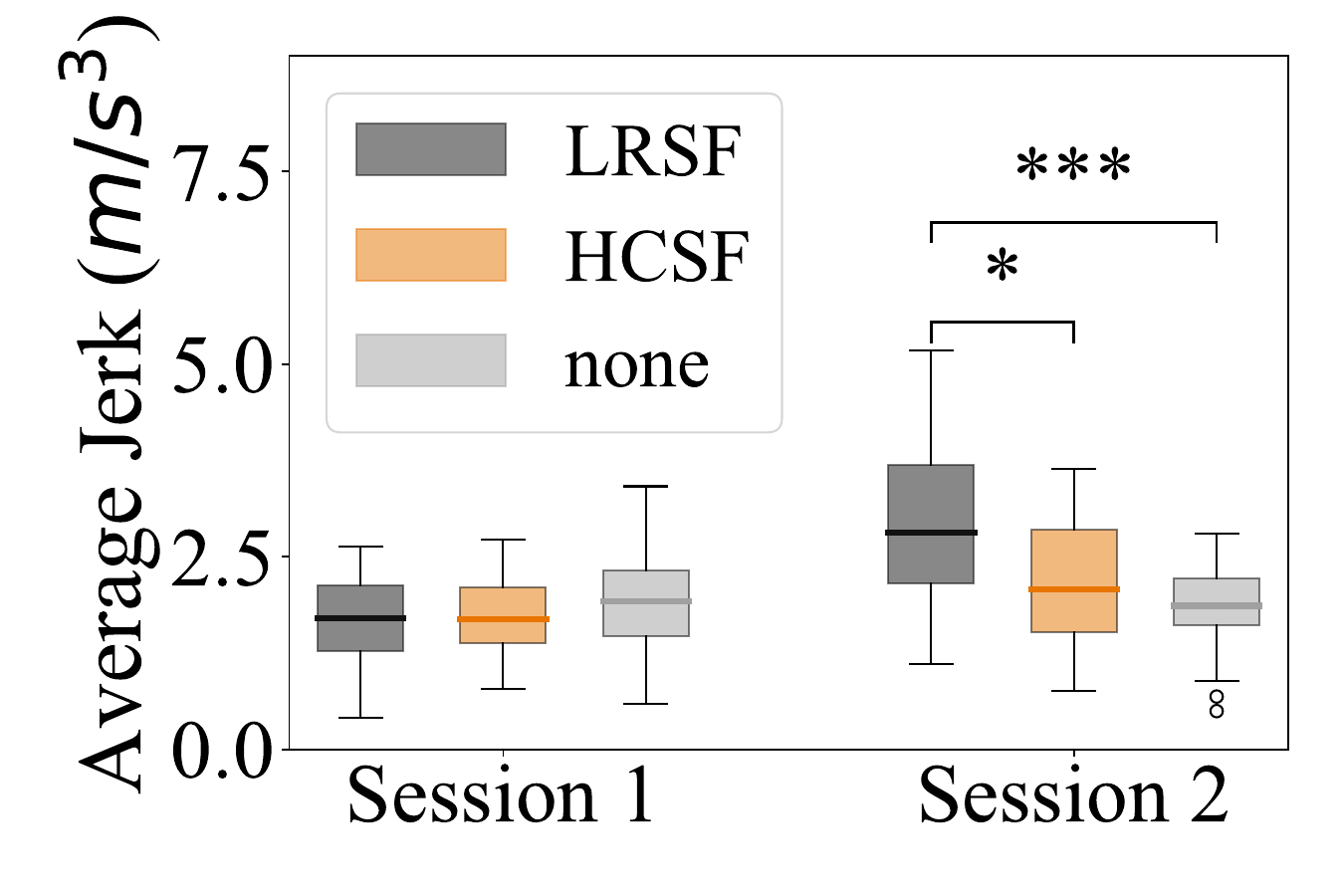}
        \caption{Box plot of average jerk.}
        \label{fig:avg_jerk_box}
    \end{subfigure}
    \caption{Our \gls{HCSF} produced significantly smaller jerk compared to \gls{LRSF}, suggesting better ride comfort, while its difference from unassisted driving was not statistically significant.}
    \label{fig:smoothness_jerk}
\end{figure}

\begin{figure}[!h]
    \centering
    \includegraphics[width=\columnwidth]{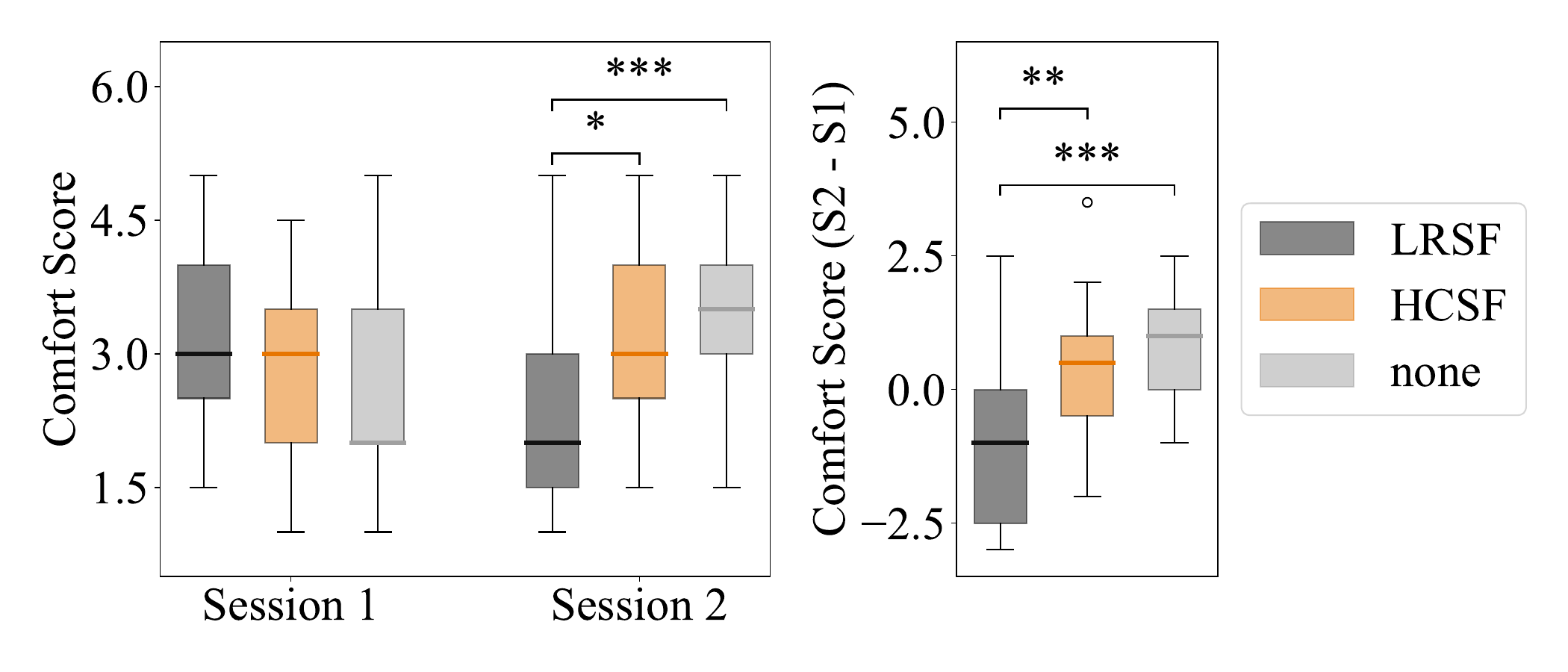}
    \caption{Qualitative measures of comfort across all three groups in both sessions. Participants in the \gls{HCSF} group reported a significantly better sense of smoothness than those in the \gls{LRSF} group, while the difference between the \gls{HCSF} and unassisted groups was not statistically significant.}
    \label{fig:box_s1s2_and_delta_comfort} 
\end{figure}

\subsubsection{Satisfaction}
Fig.~\ref{fig:box_s1s2_and_delta_satisfaction} shows that during session 2, participants who received our \gls{HCSF} assistance reported significantly higher overall satisfaction scores compared to those who received \gls{LRSF} assistance or no assistance. This difference in user satisfaction between the \gls{HCSF} and \gls{LRSF} groups is also evident when examining each participant’s change in scores from session 1 to session 2.

Therefore, we conclude: 
\begin{itemize} 
    \item Our \gls{HCSF} significantly improves user satisfaction compared to having no filter. 
    \item Our \gls{HCSF} significantly improves user satisfaction compared to \gls{LRSF}. 
\end{itemize}

\begin{figure}[!h]
    \centering
    \includegraphics[width=\columnwidth]{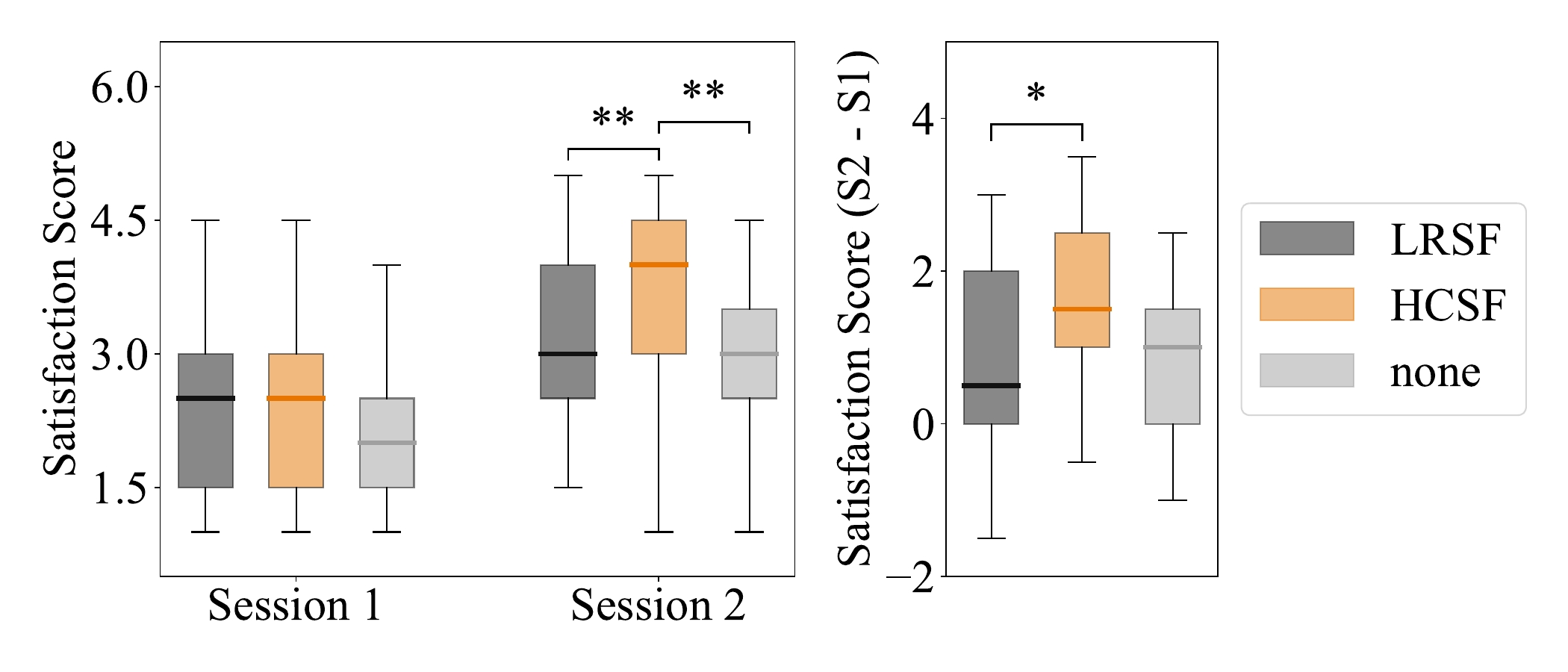}
    \caption{Qualitative measures of overall satisfaction across all three groups in both sessions. Participants in the \gls{HCSF} group reported significantly higher satisfaction during session 2 compared to both the \gls{LRSF} and unassisted groups. }
    \label{fig:box_s1s2_and_delta_satisfaction} 
\end{figure}

\subsubsection{Filter-Specific Metrics}
We further analyze the interplay between human drivers and the safety filter using four filter-specific metrics---trustworthiness, predictability, interpretability, and competence. The participant responses, shown in Fig.~\ref{fig:ai_specific_session2}, reveal that our \gls{HCSF} exhibits significantly higher trustworthiness and competence compared to unassisted driving. Meanwhile, \gls{LRSF} falls between our \gls{HCSF} and unassisted driving without a statistically significant difference. Additionally, participants rated \gls{LRSF} as significantly more unpredictable compared to unassisted driving.

Although the quantitative robustness measures in Fig.~\ref{fig:safety_metrics} show no significant difference between our \gls{HCSF} and \gls{LRSF}, we conjecture that participants perceived \gls{LRSF} as less trustworthy and competent due to its disregard for human input, which results in discontinuous, jerky, and unpredictable interventions. Unable to anticipate when or how \gls{LRSF} would intervene, participants were less inclined to trust the system. In contrast, because our \gls{HCSF} minimizes deviations from the human actions while still preserving robustness, it may have felt more trustworthy and competent to the users.

\begin{figure}[!h]
    \centering
    \includegraphics[width=\columnwidth]{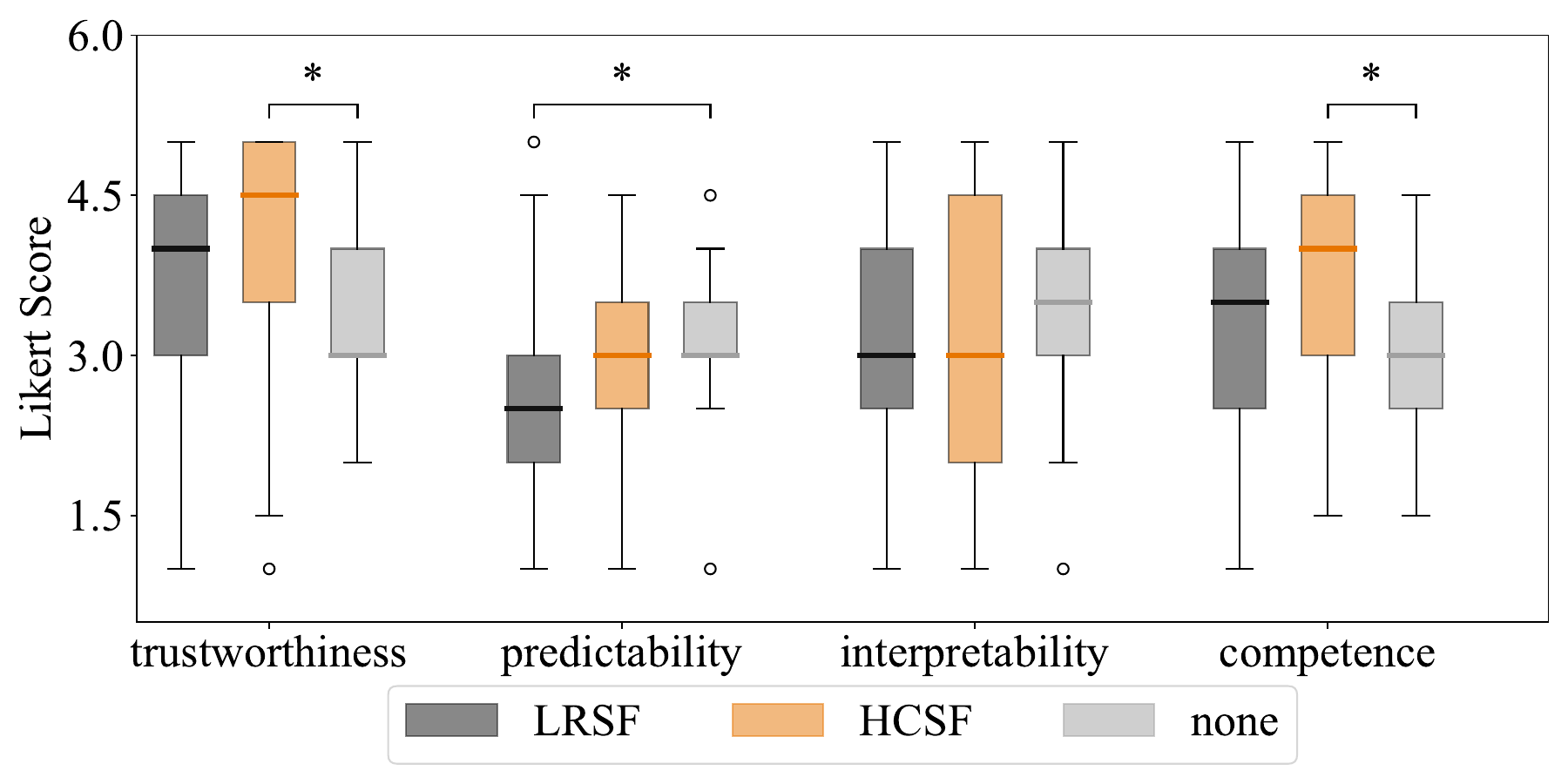}
    \caption{Qualitative measures of filter-specific metrics---trustworthiness, predictability, interpretability, and competence---across all three groups in session 2. Central marks, bottom, and top edges of the boxes indicate the median, 25th, and 75th percentiles, respectively. The maximum whisker length is set to $1.5$ times the standard deviation, which gives $99.7$ percent coverage for normally distributed data. Statistical significance is marked with asterisks, where more asterisks indicate larger significance.}
    \label{fig:ai_specific_session2} 
\end{figure}

\section{Limitations and Future Work}
\label{sec:limitation}

Our proposed \gls{HCSF} offers significant improvements in human agency, comfort, and satisfaction compared to conventional safety filters; however, it does have several limitations.
First, our \gls{HCSF} does not directly optimize for racing performance (\eg, lap times), as its sole objective is to minimize deviation from the human input while satisfying the \gls{Q-CBF} constraint. 
Although our \gls{HCSF} offers a slight performance gain over \gls{LRSF} and unassisted driving, the difference in lap times was not statistically significant (Appendix \ref{app:Results}).
Second, prolonged exposure to our \gls{HCSF} may result in drivers becoming \emph{over-reliant} on AI assistance, potentially hindering the development of unassisted driving skills (Appendix \ref{app:Results}).
Third, the inherently reactive design of our \gls{HCSF} addresses only immediate safety threats, leaving drivers unprepared for complex race dynamics that require proactive anticipation and adaptation.
Finally, while visual cues can help reduce confusion during safety interventions, they may also induce unintended behavioral changes that merit further investigation.

To address these limitations, future work could add a secondary intervention layer that activates prior to safety filter interventions and provides proactive, multi-modal cues instead of removing the human operator's control authority.
We expect this new layer to mitigate the potential over-reliance on safety filters because it does not intervene at the physical control channels shared between a human operator and a safety filter; instead, it ``coaches'' the operator to take appropriate actions before a safety intervention is needed.
Moreover, since we can rely on a safety filter to maintain system safety, the new layer can be designed to foster performance-oriented operation.
Building this additional layer will require further investigation into how humans respond to different modalities and content of cues.
Overall, these advancements would yield a more comprehensive approach to balancing safety and performance in shared autonomy environments.

\section{Conclusion} 
\label{sec:conclusion}
We proposed an \gls{HCSF} that significantly enhances system safety while preserving human agency in human--AI shared autonomy settings. 
Our \gls{HCSF} builds upon a neural safety value function which is learned scalably through black-box interactions via model-free \gls{RL}-based \gls{HJ} reachability analysis. At deployment, we used this value function to enforce a novel \gls{Q-CBF} safety constraint, which does not require any knowledge of the system dynamics for safety monitoring and intervention.
These properties enabled both the synthesis and deployment of our \gls{HCSF} in Assetto Corsa---a high-fidelity black-box car racing simulator---and make our method the first safety filter applied to high-dimensional, black-box shared autonomy systems involving human operators. 
Through an extensive in-person user study, we validated two hypotheses using both trajectory data and user responses: 
\begin{itemize}
    \item \textbf{H1}: Relative to unassisted driving, our proposed \gls{HCSF} improves \textbf{safety} and user \textbf{satisfaction} without compromising human \textbf{agency} and \textbf{comfort}.
    \item \textbf{H2}: Compared to a conventional safety filter, our proposed \gls{HCSF} improves human \textbf{agency}, \textbf{comfort}, and \textbf{satisfaction} without compromising \textbf{robustness}.
\end{itemize}

We envision our \gls{HCSF} being a vital component for a broad class of shared autonomy systems, including advanced driver assistance systems (ADAS) in high-performance driving, thanks to our \gls{HCSF}'s model-free and scalable design. Unlike conventional safety filters---where abrupt interventions that do not take into account the human operator's actions can cause discomfort or automation surprise---our \gls{HCSF} fosters trustworthy, robust, and agency-preserving human--robot collaboration.
Future work will extend our \gls{HCSF} to time-critical and performance-oriented tasks focusing on the interplay between safety and strategic decision-making, and address the potential over-reliance on safety filters.

\balance
\bibliographystyle{IEEEtran}
\bibliography{references.bib}

\newpage
\appendix

\subsection{Proofs of Theoretical Results} \label{app:Proofs}

In this section, we present the proofs for \autoref{pro:model-free_DCBF} and \autoref{pro:recursive_feasibility}. We first state two key lemmas.

\begin{lemma}
\label{lemma:UBQfunc}
For all $\state \in \safeset^*$ and for all $\gamma \in [0, 1)$, there exists an action $\ctrl \in \cset$ such that $\qfunc(\state, \ctrl) \geq \gamma V(x)$.
\end{lemma}

\begin{proof}
There exists $\ctrl \in \cset$ that satisfies $\qfunc(\state, \ctrl) = \valfunc(\state)$ for all $\state\in\safeset^*$. One such action is from the safe fallback policy, $\fallback(\state)\coloneq\argmaxB_{\ctrl\in\cset}\qfunc(\state, \ctrl)$. For any $\gamma \in [0, 1)$ we have $\gamma \valfunc(\state) \leq \valfunc(\state)$. Therefore, there always exists an action $\ctrl \in \cset$ such that $\qfunc(\state, \ctrl) \geq \gamma \valfunc(\state)$ for all $\gamma \in [0, 1)$.
\end{proof}

\begin{lemma}
\label{lemma:VfuncequivalenceFromUBQ}
For all $\state \in \safeset^*$ and for all $\gamma \in [0, 1)$, $\qfunc(\state, \ctrl) \geq \gamma \valfunc(\state)$ is equivalent to $\valfunc(f(\state, \ctrl)) - \valfunc(\state) \geq -\alpha \valfunc(\state)$, where $\alpha = 1 - \gamma$.
\end{lemma}

\begin{proof}
Subtracting $\valfunc(\state)$ from both sides of $\qfunc(\state, \ctrl) \geq \gamma \valfunc(\state)$ results in $\qfunc(\state, \ctrl) - \valfunc(\state) \geq (\gamma - 1)\valfunc(\state)$.
Using the definition of the $\qfunc$-function from~\eqref{eq.state_action_safety_bellman_eq} we substitute for $\qfunc(\state, \ctrl)$:
\begin{equation*}
\min \{ g(\state), \max_{\ctrl' \in \cset} \qfunc(f(\state, \ctrl), \ctrl') \} - \valfunc(\state) \geq (\gamma - 1)\valfunc(\state).
\end{equation*}
This inequality holds if and only if both of the following conditions are satisfied:
\begin{align*}
g(\state) - \valfunc(\state) \geq (\gamma - 1)\valfunc(\state), \quad \text{(Condition 1)} \\
\max_{\ctrl' \in \cset} \qfunc(f(\state, \ctrl), \ctrl') - \valfunc(\state) \geq (\gamma - 1)\valfunc(\state). \quad \text{(Condition 2)}
\end{align*}
Analyzing Condition 1: from \eqref{eq.safety_bellman_eq}, we have $g(\state) - \valfunc(\state) \geq 0$.
$\gamma \in [0, 1)$ implies $\gamma - 1 < 0$ and $\valfunc(\state)\geq0$ for all $\state\in\safeset^*$. Therefore, $g(\state) - \valfunc(\state) \geq (\gamma - 1)\valfunc(\state)$ is always true.

\noindent Analyzing Condition 2: from~\eqref{eq.state_action_safety_bellman_eq}, $\valfunc(f(\state, \ctrl)) = \max_{\ctrl' \in \cset} \qfunc(f(\state, \ctrl), \ctrl')$. Substituting this into Condition 2, we get $\valfunc(f(\state, \ctrl)) - \valfunc(\state)\geq (\gamma - 1)\valfunc(\state)$.

Hence, for all $\state\in\safeset^*$ and for all $\gamma=1-\alpha\in[0, 1)$, $\qfunc(\state, \ctrl) \geq \gamma \valfunc(\state)$ is true if and only if $\valfunc(f(\state, \ctrl)) - \valfunc(\state) \geq -\alpha \valfunc(\state)$ is true. 
\end{proof}

We are now ready to prove \autoref{pro:model-free_DCBF} and \autoref{pro:recursive_feasibility}. We restate them here for convenience. 

\begin{proposition*}[Restatement of \autoref{pro:model-free_DCBF}]
The safety value function $\valfunc(\state):\xset\rightarrow\reals$, which is a fixed-point solution of the safety Bellman equation \eqref{eq.safety_bellman_eq}, is a valid \gls{DCBF} as defined in \autoref{def.DCBF} and \autoref{rem.DCBF_relax}. The corresponding \gls{DCBF} constraint is:
\begin{equation}
\label{eq.model_free_DCBF_constraint_appendix}
\qfunc(\state, \ctrl)\geq\gamma\valfunc(\state),
\end{equation}
where $\gamma\in[0, 1)$.
\end{proposition*}
\begin{proof}
To prove that $\valfunc(\state)$ is a valid DCBF, it is sufficient to show that for all $\state \in \safeset^*$, there exists an action $\ctrl \in \cset$ that satisfies
\begin{equation}
\label{eq.valfunc_DCBF}
\valfunc(f(\state, \ctrl)) - \valfunc(\state) \geq -\alpha \valfunc(\state)
\end{equation}
for some $\alpha \in (0, 1]$. By \autoref{lemma:VfuncequivalenceFromUBQ}, \eqref{eq.valfunc_DCBF} is equivalent to \eqref{eq.model_free_DCBF_constraint_appendix} with $\alpha=1-\gamma$. Moreover, by \autoref{lemma:UBQfunc}, there exists an action $\ctrl\in\cset$ that satisfies \eqref{eq.model_free_DCBF_constraint_appendix} for all $\state\in\safeset^*$ and for all $\gamma\in[0, 1)$. This equivalence establishes that the $\qfunc$-function directly defines the DCBF constraint in a model-free manner, replacing the conventional dynamics-based constraint.
\end{proof}

\begin{proposition*}[Restatement of \autoref{pro:recursive_feasibility}]
The optimization problem in \eqref{eq.HCSF} is recursively feasible for $\forall\gamma \in [0, 1)$, given the initial state $\state \in \safeset^*$.
\end{proposition*}
\begin{proof}
By \autoref{lemma:UBQfunc}, for all $\state \in \safeset^*$ and $\gamma \in [0, 1)$, $\exists \ctrl \in \cset$ such that $\qfunc(\state, \ctrl) \geq \gamma \valfunc(\state) \implies \qfunc(\state,\ctrl) \geq 0$. 
From~\eqref{eq.state_action_safety_bellman_eq}, this choice of $\ctrl$ ensures $\valfunc(f(\state, \ctrl)) \geq 0 \implies f(\state, \ctrl) \in \safeset^*$, meaning the system remains within the maximal safe set $\safeset^*$ at the next timestep. Since this reasoning applies recursively for all timesteps, the optimization problem \eqref{eq.HCSF} is feasible for $\forall\gamma \in [0, 1)$.
\end{proof}

\subsection{Additional Environment Details}
\label{app:Environment}

\subsubsection{Observation}
We use a 133-dimensional observation vector for both neural synthesis of safety filters and for training the warmup policies. Table~\ref{tab:observation} lists the elements of this observation vector, their corresponding dimensionalities, and the number of past timesteps (i.e., ``stacked'' frames) used:
\begin{table}[ht]
\centering
\renewcommand{\arraystretch}{1.2}
\caption{Hyperparameters for the Warmup \& Initialization Phases}
\begin{tabular}{c|c|c}
\hline
\textbf{Observation} & \textbf{Dimensionality} & \textbf{Past Timesteps} \\ \hline
Ego vehicle speed & 1 & 4 \\
Gap to the reference path & 1 & 4 \\
Force feedback & 1 & 4 \\
RPM & 1 & 4\\
Acceleration & 2 & 4\\
Gear & 1 & 4\\
Angular velocity & 1 & 4\\
Local velocity & 2 & 4\\
Slip Angle & 4 & 4\\
Distance to track boundary & 11 & 4\\
Out-of-track & 1 & 1\\
Look ahead curvature & 12 & 1\\
Control input & 3 & 4\\
Ego vehicle heading & 1 & 1\\
Ego vehicle yaw & 1 & 1\\
Opponent distance & 1 & 1\\
Opponent direction & 1 & 1\\
Opponent speed & 1 & 1\\
Opponent heading & 1 & 1\\
Opponent yaw & 1 & 1\\
Opponent brake & 1 & 1\\
\hline
\end{tabular}
\label{tab:observation}
\end{table}

\subsubsection{Action}
The actions represent increments to the control inputs over a single timestep rather than absolute control values. This incremental representation mitigates oscillations or chattering in the control signals \cite{remondasimulation}.

\subsection{Extended Implementation Details}
\label{app:Training}

In this section, we provide additional details about our training methodology. Refer to Table~\ref{tab:hcsf_warmup_init_hyperparams} for the hyperparameter values mentioned in this subsection.

\subsubsection{Warmup and Initialization}

The training pipeline for \gls{HCSF} is designed to promote robustness and improve learning efficiency in the high-fidelity simulation environment of \gls{AC}. It operates in three distinct phases per episode---warmup, initialization, and training---each ensuring the ego agent encounters a broad spectrum of scenarios, including failure-prone states that yield the most valuable data for training the safety filter.

\noindent
\textbf{Warmup Phase.}
\noindent
In \gls{AC}, each reset places the ego agent in a trivial, stationary configuration on the reference path, with its heading tangent to the path. Because it remains in the safe set if it does not move, this state provides little value for learning the \(\qfunc\)-function or the best-effort fallback policy---particularly regarding near-failure conditions. To address this, we introduce a \emph{warmup phase} that accelerates the agent to higher speeds and initiates interactions with the track or opponents before formal training begins. 

During warmup, the agent follows either an overtaking policy \(\policy^\warmup_\overtake\) or a nominal policy \(\policy^\warmup_\nominal\) for a duration sampled from \([0, T^{\warmup, \max}]\). These two policies are chosen with probability \(\prob^\warmup_\overtake\) and its complement, respectively, exposing the agent to both opponent-aware and opponent-agnostic driving scenarios. This warmup procedure transitions the agent away from trivial reset conditions and into more realistic states where the safety filter’s intervention is most meaningful. For additional details regarding the training of the warmup policies, see the next subsection.

To diversify starting states, the warmup phase may terminate early when the ego agent enters failure-prone regions, with probability \(\prob^\warmup_\opponent\) if it is within a distance sampled from \([d^{\warmup, \min}_\opponent,\, d^{\warmup, \max}_\opponent]\) of the nearest opponent. This increases the frequency of interactions with opponents and helps in learning collision avoidance strategies. Additionally, heavy braking can trigger early termination: at the start of each episode, a gating probability \(\prob^\warmup_\text{brake, epi}\) determines whether braking is considered for termination. If allowed, termination occurs with probability \(\prob^\warmup_\text{brake, step}\) at each timestep when the braking input exceeds \(\ctrl^\warmup_\brake\) and speed is above \(v^\warmup\). Employing these probabilities sequentially ensures that heavy braking only leads to termination under conditions resembling realistic failure scenarios. This phased termination approach helps the agent reach states critical for learning an effective safety filter.

\noindent
\textbf{Initialization Phase.}
While the warmup phase introduces motion and some variability in initial conditions, it relies primarily on well-trained policies like \(\policy^\warmup_\overtake\), which tend to stay on the reference path and avoid collisions. As a result, it rarely induces suboptimal driving behaviors or pushes the ego agent into failure-prone scenarios. This limitation is problematic because effective learning of the \(\qfunc\)-function and the best-effort fallback policy demands spanning the full state space \(\xset\) and control space \(\cset\) \cite{fisac2019bridging}.

To address this, we introduce the \emph{initialization phase}, which systematically exposes the ego agent to failure-prone regions of \(\xset\) and \(\cset\). At each timestep, the agent selects an action and queries the \(\qfunc\)-function to check if the resulting state--action pair falls below \(\qfunc^\initialization_\text{term}\), indicating a potentially hazardous scenario. With probability \(\prob^\initialization_\text{term}\), the initialization phase terminates, and the ego vehicle can start the training phase in this ``dangerous'' state where safety intervention is most critical.
Because human drivers often miss braking points before corners and maintain throttle when braking is required, the initialization phase also simulates a \emph{full-throttle mode} with probability \(\prob^\initialization_\text{FT}\). We define \(\mathcal{U}_\text{FT} = \{ (\ctrl_\steer, \ctrl_\throttle, \ctrl_\brake) \mid \ctrl_\steer \in [-1,1], \ctrl_\throttle = 1, \ctrl_\brake = -1 \}\), representing high-speed or aggressive control signals aimed at inducing failure-prone conditions stemming from late braking.

Three initialization schemes---\emph{adversarial}, \emph{random}, and \emph{mixed}---provide distinct ways to sample actions:
\begin{itemize}
    \item \textbf{Adversarial Initialization:}  
    Chosen with probability \(\prob^\initialization_\text{adv}\). Actions come from either \(\mathcal{U}_\text{FT}\) (if full-throttle mode is engaged) or \(\cset\). The \(\qfunc\)-function identifies the action with the smallest \(\qfunc\)-value (the ``adversarial'' action). If its value lies below \(\qfunc^\initialization_\text{term}\), the phase terminates with probability \(\prob^\initialization_\text{term}\). This repeats until termination.

    \item \textbf{Random Initialization:}  
    Chosen with probability \(\prob^\initialization_\text{rand}\). Actions are sampled randomly from \(\mathcal{U}_\text{FT}\) or \(\cset\), without querying \(\qfunc\). If the \(\qfunc\)-value for a sampled action is below \(\qfunc^\initialization_\text{term}\), termination occurs with probability \(\prob^\initialization_\text{term}\).

    \item \textbf{Mixed Initialization:}  
    Chosen with probability \(\prob^\initialization_\text{mix}\). At each timestep, the agent switches between adversarial or random strategies with probability \(0.5\). Termination follows the same rule, triggered by \(\prob^\initialization_\text{term}\) whenever the \(\qfunc\)-value is below \(\qfunc^\initialization_\text{term}\).
\end{itemize}

This approach helps the system generalize to real deployment scenarios, where human operators may have varying skill levels and intentions, leading to unpredictable behaviors and frequent exposure to failure-prone conditions.

\begin{table*}[ht]
\centering
\renewcommand{\arraystretch}{1.2}
\caption{Hyperparameters for the Warmup \& Initialization Phases}
\begin{tabular}{c|c|l}
\hline
\textbf{Hyperparameter} & \textbf{Value} & \textbf{Description} \\ \hline
$T^{\warmup, \max}$ & $25\text{ s}$ & Max duration of warmup phase. \\
$\prob^\warmup_\overtake$ & $0.6$ & Probability of using overtaking policy during warmup. \\
$\prob^\warmup_\opponent$ & $0.25$ & Probability of early termination if near an opponent. \\
$d^{\warmup, \min}_\opponent$ & $6\text{ m}$ & Min distance threshold for opponent proximity. \\
$d^{\warmup, \max}_\opponent$ & $36\text{ m}$ & Max distance threshold for opponent proximity. \\
$\prob^\warmup_\text{brake, epi}$ & $0.4$ & Probability that heavy braking triggers warmup termination. \\
$\prob^\warmup_\text{brake, step}$ & $0.25$ & Probability of termination on each timestep of heavy braking. \\
$\ctrl^\warmup_\brake$ & $0.6$ & Threshold for heavy braking input. \\
$v^\warmup$ & $40~\text{m/s} $ & Speed threshold under which braking triggers termination. \\ 
$\prob^\initialization_\text{term}$ & $0.2$ & Probability of ending initialization once $\qfunc$-value falls below threshold. \\
$\qfunc^\initialization_\text{term}$ & $2$ & $\qfunc$-value threshold for indicating dangerous scenarios. \\
$\prob^\initialization_\text{FT}$ & $0.4$ & Probability of enabling full-throttle mode. \\
$\prob^\initialization_\text{adv}$ & $0.3$ & Probability of adversarial initialization. \\
$\prob^\initialization_\text{rand}$ & $0.3$ & Probability of random initialization. \\
$\prob^\initialization_\text{mix}$ & $0.4$ & Probability of mixed initialization. \\
$c_1$ & $1/12$ & Design parameter for training the nominal warmup policy. \\
$c_2$ & $300$ & Design parameter for training the nominal warmup policy.\\
$d^\overtake_\opponent$ & $100\text{ m}$ & Distance threshold for overtaking rewards. \\
$c_3$ & $600$ & Design parameter for training the overtaking policy. \\
\hline
\end{tabular}
\label{tab:hcsf_warmup_init_hyperparams}
\end{table*}

\subsubsection{Training the Warmup Policies} We employ two different warmup policies during the warmup phase: a \emph{nominal} policy and an \emph{overtaking} policy. The nominal policy, trained using the reward function from \cite{remondasimulation}, seeks the fastest lap times without considering opponents. In contrast, the overtaking policy accounts for the nearest opponent, rewarding successful overtakes and penalizing collisions, although it does not guarantee collision avoidance. The nominal policy is trained using the following reward function:
\begin{equation}
    \label{eq.nominal_reward}
    r_\nominal 
    = \frac{v}{c_2}\,\bigl(1 - c_1 \cdot d_\text{gap}\bigr),
\end{equation}
where \(v\) denotes the speed of the ego vehicle, \(d_\text{gap}\) is the \(\ell_2\)-distance to the reference path, and \(c_1\) and \(c_2\) are design parameters.

For the overtaking policy, we encourage overtaking maneuvers using a term inspired by \cite{wurman2022outracing}:
\begin{equation}
    r_{\overtake, 1} = \text{I}\{d_\opponent<d^\overtake_\opponent\} \cdot c_3\cdot (\Delta\text{NSP}_{t-1}-\Delta\text{NSP}_t),
\end{equation}
where ``oppo'' indicates the nearest opponent, \(\text{I}\{\cdot\}\) is an indicator function that equals \(1\) if \(d_\opponent < d^\overtake_\opponent\) and \(0\) otherwise, \(c_3\) is a design parameter, and \(\Delta\text{NSP}_t \coloneq \text{NSP}_{\text{oppo}, t} - \text{NSP}_{\text{ego}, t}\). Here, ``NSP'' (normalized spline position) measures progression along the track, taking values in \([0, 1)\). By construction, \(r_{\overtake, 1} > 0\) if and only if either the ego vehicle is behind the opponent and is closing in, or the ego vehicle is ahead of the opponent and is pulling away.
Overtaking is further encouraged with an additional term:
\begin{equation}
    r_{\overtake, 2} = \text{I}\{r_{\overtake, 1}>0\}\cdot \frac{c_1}{c_2}\cdot v\cdot d_{\text{gap}},
\end{equation}
which activates only if the ego vehicle is closing in on the opponent from behind or pulling away in front. In effect, this negates the penalty term for deviating from the reference path in \eqref{eq.nominal_reward}, allowing the ego agent to set up or complete an overtake without strictly adhering to the reference path.
The entire reward function for training the overtaking policy is as follows:
\begin{equation}
\label{eq.overtaking_policy_reward}
    r_\overtake=r_\nominal+\max(r_{\overtake, 1}, 0)+r_{\overtake, 2}.
\end{equation}
We also note that, when training the overtaking policy, the ego vehicle is heavily penalized for colliding with the opponent through immediate episode termination.

\subsubsection{Training Details} Following \cite{hsu2023isaacs, fisac2019bridging}, we approximate the state--action safety value function \(\qfunc(\cdot, \cdot)\) using a neural network with parameters \(\phi\). We also parameterize the best-effort fallback policy as a separate neural network with parameters \(\theta\). Our goal is for \(\qfunc_\phi\) to approximate a fixed-point solution of the time-discounted safety Bellman equation, defined as:
\begin{equation} \label{eq.discounted_safety_bellman_eq}
    \learnedvfunc(x) = (1-\envdiscount)g(x) + \envdiscount \min\{g(x), \max_u \learnedqfunc(x,u))\}.
\end{equation}
To remain compatible with standard reinforcement learning methods such as \gls{SAC}~\cite{haarnoja2018soft}, we train the $\qfunc$-network to minimize the Bellman residual:
\begin{equation} \label{eq.isaacs_valueloss}
    L(\phi) = \mathbb E^{\buffer, \pi} [ (Q_\phi(x, u) - y) ^2], 
\end{equation}
where \(y := (1-\envdiscount)\,g' \;+\; \envdiscount \,\min\!\bigl\{g',\,Q_\phi\bigl(x',\,u'\bigr)\bigr\}\), and \(\qtargetnetwork\) is the target \(\qfunc\)-network. We then update \(\learnedpolicy\) using the following policy gradient:
\begin{equation} \label{eq.isaacs_policyloss}
    L(\theta) := \mathbb E^{\buffer, \pi} [-\learnedqfunc(x, u) + \alpha \log \learnedpolicy(u|x)].
\end{equation}
\eqref{eq.isaacs_valueloss} and \eqref{eq.isaacs_policyloss} are updated in an alternating manner by sampling from \(\buffer\). We employ double \(\qfunc\)-learning and update the \(\qfunc\)-networks with a delay. Table~\ref{tab:sac_hyperparameters} summarizes our hyperparameters. The complete \gls{HCSF} training pseudocode is given in \autoref{alg.hcsf_training}. 

\begin{table}[h!]
\centering
\renewcommand{\arraystretch}{1.2}
\caption{Hyperparameters for \gls{SAC} Training}
\begin{tabular}{c|c|l}
\hline
\textbf{Hyperparameter} & \textbf{Value} & \textbf{Description} \\ \hline
$\eta$ & $3 \times 10^{-4}$ & Actor and critic learning rate. \\ 
$\envdiscount$ & $0.992$ & Discount factor. \\ 
$|\mathcal{B}|$ & $2\times10^7$ & Replay buffer size. \\ 
Batch Size & $256$ & Training batch size. \\ 
$\tau$ & $0.005$ & Target network update rate. \\ 
$\alpha$ & Learned & Entropy temperature. \\ 
$N_{UTD}$ & $1$ & Gradient steps per environment step. \\ 
\hline
\end{tabular}
\label{tab:sac_hyperparameters}
\end{table}

\subsubsection{Solving the OCP} We solve the \gls{OCP} \eqref{eq.HCSF} by sampling 2000 candidate actions on the line that connects the human's action $u^\human(x)$ and the best-effort fallback policy's action $u^\shield(x)$ in the control space. Among the candidate actions, we select the one that minimizes the deviation from $u^\human(x)$ while satisfying \eqref{eq.HCSF_constraint} (note that $u^\shield(x)$ always satisfies \eqref{eq.HCSF_constraint}). This is because solely relying on the learned $\qfunc$-function to solve the \gls{OCP} \eqref{eq.HCSF}, such as sampling in the entire control space and searching for the candidate that minimizes the deviation from $u^\human(x)$ while satisfying \eqref{eq.HCSF_constraint}, led to unsatisfactory safety metrics in practice. We compensate for the neural approximation error in the $\qfunc$-function by additionally leveraging the information encoded in the best-effort fallback policy $u^\shield(x)$ \cite{hsu2023isaacs}, thereby achieving near-zero failures throughout the experiment. Optimization remains computationally efficient, as candidate actions are rapidly propagated through the Q-network using GPU-accelerated tensor computations. The pseudocode for \gls{HCSF} optimization is given in \autoref{alg.hcsf_execution}.

\subsubsection{Choice of Design Parameters}
$\gamma$ is a design parameter in our \gls{HCSF} formulation \eqref{eq.HCSF} that does \emph{not} affect the \emph{safety guarantees} but does influence the shared autonomy system's \emph{behavior}. In related works \cite{borquez2024safety}, $\gamma$ is tuned for objectives that are not directly related to safety. In our \gls{HCSF} framework, we tune $\gamma$ to modulate human agency, comfort, and satisfaction. Based on a preliminary trial among the authors, we determined that $\gamma=0.7$ gives a strong sense of agency and satisfaction while feeling smooth, and we used this value throughout the user study.

We further analyze $\gamma$'s effect on agency and comfort via an ablation study in which the overtaking policy $\pi^\warmup_\overtake$ that was trained using the reward function \eqref{eq.overtaking_policy_reward} raced the ego vehicle against an opponent in lieu of human drivers. We tested 10 instantiations of our \gls{HCSF} with different $\gamma$ values ranging from 0 to 0.9, and each instantiation was tested for 10 minutes. We note that this setting is identical to driving session 2 of our user study. The average I.M., which is a proxy for the agency metric, and average I.D. and jerk, which are proxies for comfort, for each \gls{HCSF} instantiation are shown in Fig.~\ref{fig:metrics_vs_gamma}. Smaller I.M. indicates better agency, while smaller I.D. and jerk indicate better comfort. We observe that as $\gamma$ becomes larger, proxies for comfort improve while the proxy for agency degrades. This is coherent with the understanding of $\alpha=1-\gamma$ in conventional \gls{DCBF} settings \eqref{eq.DCBF_constraint}.

We define the \emph{Combined Metric} as the uniformly weighted sum of I.M., I.D., and jerk after linearly normalizing them to be between 0 and 1. As shown in Fig.~\ref{fig:combined_metric_vs_gamma}, we found $\gamma=0.7$ to be a sweet spot between agency and comfort, a result coinciding with our preliminary trial. 
\begin{figure}[h!]
    \centering
    \begin{subfigure}[b]{0.49\columnwidth}
        \centering
        \includegraphics[width=\textwidth]{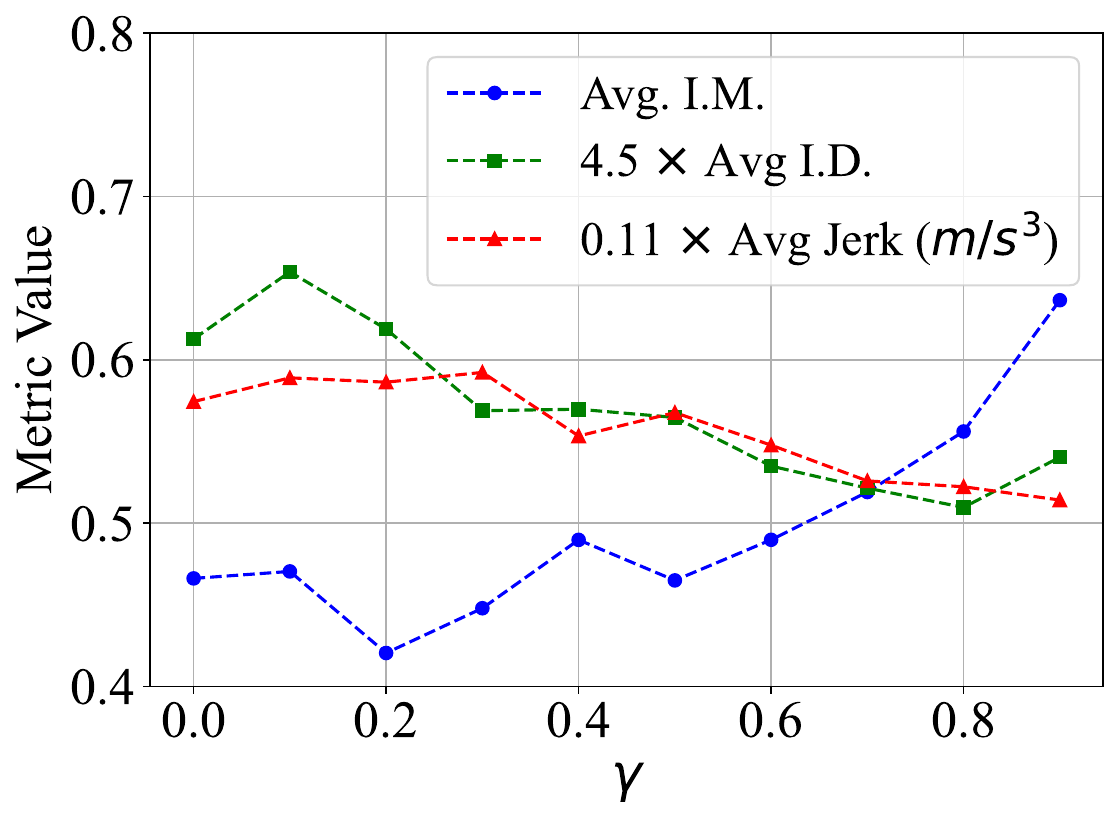}
        \caption{Average I.M., I.D., and jerk.}
        \label{fig:metrics_vs_gamma}
    \end{subfigure}
    \begin{subfigure}[b]{0.49\columnwidth}
        \centering
        \includegraphics[width=\textwidth]{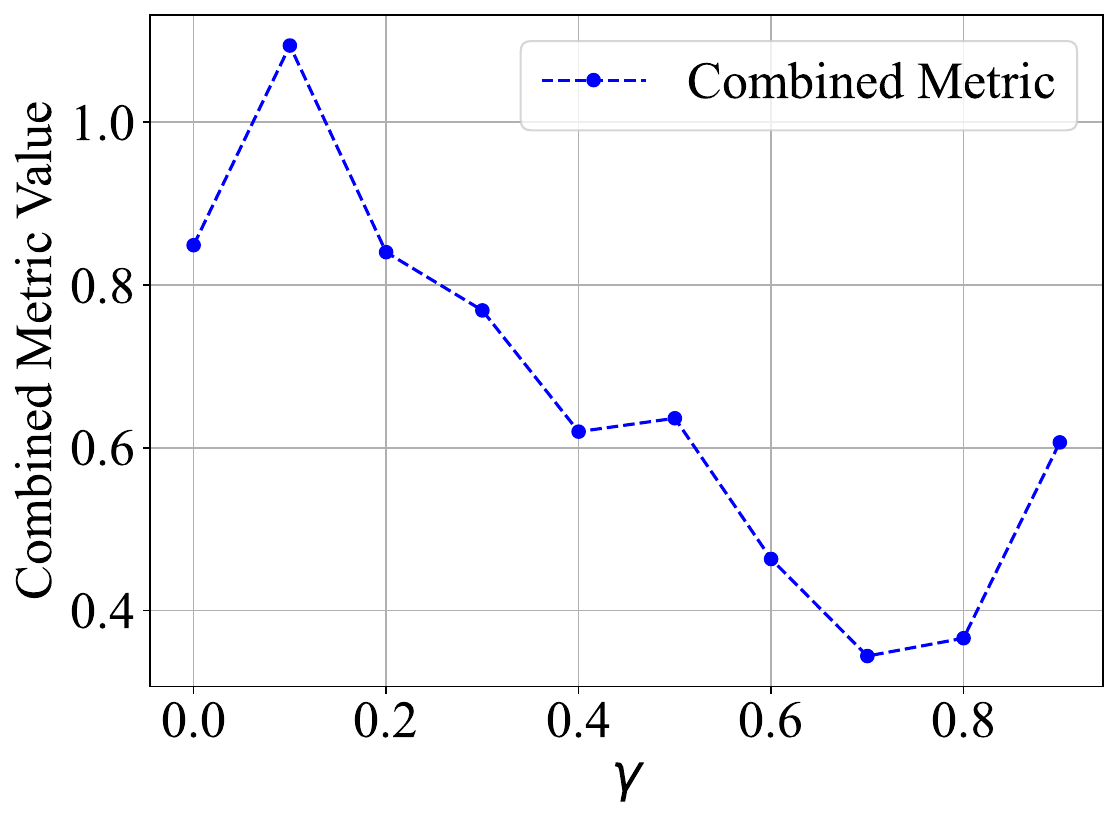}
        \caption{Combined Metric on different $\gamma$.}
        \label{fig:combined_metric_vs_gamma}
    \end{subfigure}
    \caption{\gls{HCSF} with $\gamma=0.7$ strikes a good balance between the proxies for agency and comfort.}
    \label{fig:gamma_ablation}
\end{figure}

Our framework can also integrate approaches like Parametric CBFs \cite{lyu2022adaptive} that co-optimize $\gamma$ to further enhance human agency, comfort, and satisfaction without compromising safety. We acknowledge the relevance of our work to the smooth least restrictive safety filter~\cite{borquez2024safety}, which is an instantiation of our \gls{HCSF} using $\gamma=0$ and does not necessarily strike a good balance between agency and comfort, as can be seen in Fig.~\ref{fig:gamma_ablation}. 

\subsubsection{Wall-Clock Training Time}  The three-week training time stems primarily from the AC simulator’s real-time constraint, with a majority of the time spent on warmup and initialization phases. During the training phase, which begins after the warmup and initialization phases, we sample transitions from a single environment and append them to $\buffer$.

\subsection{Extended Results}
\label{app:Results}

\subsubsection{Cronbach's Alpha Test}

Table~\ref{tab:Cronbach_alpha} summarizes the Cronbach’s alpha values obtained for each of the eight metrics. Agency (0.716), comfort (0.775), and satisfaction (0.710) exceed the commonly cited threshold of 0.70, indicating acceptable internal consistency. Robustness (0.673) and trustworthiness (0.683) fall just below 0.70 and can be considered borderline acceptable, which is reasonable given that each metric was assessed using only two items. Interpretability (0.562) and competence (0.540) are borderline, suggesting that the negated items may not perfectly capture the reverse of their affirmative counterparts. Predictability (0.175) stands out as poor, indicating a potential mismatch between its affirmative and negated statements or participants’ interpretations.

Despite this variability, the four core metrics---robustness, agency, comfort, and satisfaction---either exceed or closely approach the 0.70 threshold, making them sufficiently reliable for supporting our core hypotheses. Hence, we consider the qualitative data for these metrics to be reliable enough to validate the conclusions drawn from our study.

\begin{table}[h!]
\centering
\renewcommand{\arraystretch}{1.2}
\caption{Cronbach's Alpha Test Results}
\begin{tabular}{c|c}
\hline
\textbf{Metric} & \textbf{Cronbach's Alpha Value}\\ \hline
Robustness & \textbf{$0.673$} \\ 
Agency & \textbf{$0.716$} \\ 
Comfort & \textbf{$0.775$} \\ 
Satisfaction & \textbf{$0.710$} \\ 
Trustworthiness & \textbf{$0.683$} \\ 
Predictability & $0.175$ \\ 
Interpretability & $0.562$ \\ 
Competence & $0.540$ \\
\hline
\end{tabular}
\label{tab:Cronbach_alpha}
\end{table}

\subsubsection{Racing Performance}
The final lap times for each participant in all three sessions---1, 2, and 3---are shown in Fig.~\ref{fig:final_lap_time_bar_3sessions}. We use final lap time as a quantitative performance metric, particularly for session 2 in which participants receive a safety filter (including a placebo). We assume that participants need some time to learn how to cooperate with the assistance. Moreover, because each failure forces the vehicle to reset---causing it to restart from a stationary position---we treat failure incidents as inherently penalized in terms of lap times.
 
Although no statistical significance was observed, participants who received a safety filter in session 2 showed greater improvements in lap times compared to those in the unassisted group. Notably, the \gls{HCSF} group in session 2 was the only one to average a sub-three-minute lap time. This advantage likely stems from \gls{HCSF} providing gradual, smoother interventions rather than the last-minute, large corrections of \gls{LRSF}. Such preemptive, smooth corrections have been linked to better performance metrics (e.g., lap times) in filter-aware motion planning literature \cite{hu2022sharp, hu2024active}.

\begin{figure}[!h]
    \centering
    \includegraphics[width=\columnwidth]{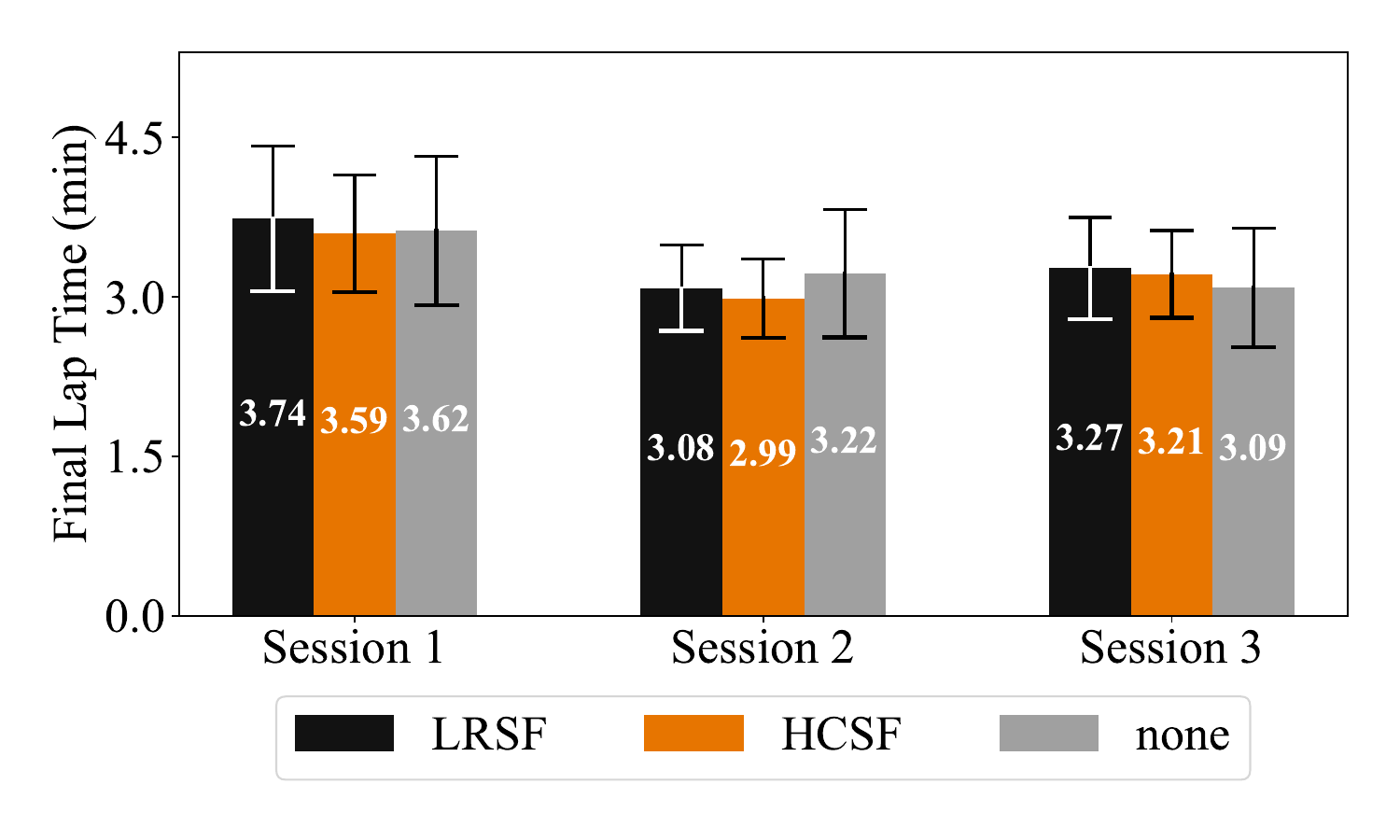}
    \caption{Average final lap times for each group across sessions 1, 2, and 3.}
    \label{fig:final_lap_time_bar_3sessions} 
\end{figure}

Although no statistical significance was observed, Fig.~\ref{fig:final_lap_time_bar_3sessions} suggests a tendency toward over-reliance on safety filters. While the average final lap time of unassisted participants decreased across sessions (likely due to gained driving experience), both \gls{LRSF} and \gls{HCSF} groups---despite showing the shortest times in session~2---had increased lap times in session~3. They became slightly slower (though not significantly) than the unassisted group, potentially because they relied too heavily on the safety filters and thus saw less improvement in driving skill. We anticipate that future \gls{HCSF} designs, which can foresee a longer time horizon and make \emph{strategic, performance-oriented} decisions beyond mere safety enforcement, may address this issue of over-reliance.

\begin{algorithm*}[h]
\caption{HCSF Training}
\begin{algorithmic}[1] \label{alg.hcsf_training}
\STATE Initialize policy parameters $\theta$, value function parameters $\phi$
\STATE Initialize replay buffer $\mathcal{D}$
\STATE Set target smoothing coefficient $\tau$, discount factor $\envdiscount$, and temperature $\alpha$
\FOR{each training episode}
    \STATE Observe initial state $x_0$
    \FOR{each environment step}
        \STATE Select action $u_t \sim \pi_\theta(u_t|x_t)$
        \STATE Execute $u_t$ in environment, observe margin $g_t$, and next state $x_{t+1}$
        \STATE Store transition $(x_t, u_t, g_t, x_{t+1})$ in replay buffer $\buffer$
    \ENDFOR
    \FOR{each gradient step $i = 1,\dots,N_{UTD}$}
        \STATE Sample a minibatch of transitions $(x_t, u_t, g_t, x_{t+1})$ from $\buffer$
        \STATE Compute target $y_t:=(1-\envdiscount)g_t + \envdiscount \min\{g_t, Q_\phi(x^\prime, u^\prime)\}$
        \STATE Update $\qfunc$-function(s) using \eqref{eq.isaacs_valueloss}
        \STATE Update policy using \eqref{eq.isaacs_policyloss}
        \STATE Update target $\qfunc$-function(s): ${\phi} \gets \tau \phi + (1 - \tau) \phi$
    \ENDFOR
\ENDFOR
\end{algorithmic}
\end{algorithm*}

\begin{algorithm*}[h]
\caption{HCSF Execution}
\begin{algorithmic}[1] \label{alg.hcsf_execution}
\FOR{each environment step}
    \STATE Observe the current state $\state$
    \STATE Observe human action $u^\human(\state)$ and compute the best-effort fallback policy's action $u^\shield(x)$
    \STATE Sample candidate actions on the line that connects $u^\human(\state)$ and $u^\shield(x)$
    \STATE Select a candidate action $\ctrl$ that minimizes \eqref{eq.HCSF_cost} while satisfying \eqref{eq.HCSF_constraint}.
    \STATE Apply the filtered output $\ctrl$ to the system
\ENDFOR
\end{algorithmic}
\end{algorithm*}

\end{document}